\documentclass{article}

\usepackage{url}
\usepackage[T1]{fontenc}
\usepackage[utf8]{inputenc}
\usepackage{microtype}
\usepackage{amsmath,amssymb,amsthm,mathtools}
\usepackage{enumitem}
\usepackage{graphicx}
\graphicspath{{figures/}}
\usepackage{booktabs}
\usepackage{tabularx}
\usepackage{caption}
\usepackage{subcaption}
\usepackage{xcolor}
\usepackage{aliascnt}
\usepackage[
  colorlinks=true,
  linkcolor=blue!55!black,
  citecolor=blue!45!black,
  urlcolor=blue!55!black
]{hyperref}
\usepackage[capitalize,noabbrev,nameinlink]{cleveref}
\usepackage{float}

\newcolumntype{Y}{>{\raggedright\arraybackslash}X}
\newcommand{\R}{\mathbb{R}}
\newcommand{\E}{\mathbb{E}}
\newcommand{\Prob}{\mathbb{P}}
\newcommand{\N}{\mathcal{N}}
\newcommand{\TV}{\mathrm{TV}}
\DeclareMathOperator{\Law}{Law}
\newcommand{\ptar}{\pi_{\mathrm{tar}}}
\newcommand{\Fhat}{\hat{F}}
\newcommand{\grad}{\nabla}
\newcommand{\eps}{\varepsilon}
\newcommand{\Id}{I_d}
\newcommand{\norm}[1]{\left\lVert #1 \right\rVert}
\newcommand{\abs}[1]{\left\lvert #1 \right\rvert}
\newcommand{\inner}[2]{\langle #1, #2 \rangle}
\newcommand{\dd}{\mathrm{d}}
\newcommand{\pihat}{\hat{\pi}}
\newcommand{\Vhat}{\hat{V}}

\newcommand{\Hloc}{\mathcal{H}_{\mathrm{loc}}}
\newcommand{\Exp}{\mathbb{E}}

\theoremstyle{plain}
\newtheorem{theorem}{Theorem}[section]

\newaliascnt{lemma}{theorem}
\newtheorem{lemma}[lemma]{Lemma}
\aliascntresetthe{lemma}

\newaliascnt{corollary}{theorem}
\newtheorem{corollary}[corollary]{Corollary}
\aliascntresetthe{corollary}

\newaliascnt{proposition}{theorem}
\newtheorem{proposition}[proposition]{Proposition}
\aliascntresetthe{proposition}

\newaliascnt{conjecture}{theorem}

\aliascntresetthe{conjecture}

\newaliascnt{assumption}{theorem}
\newtheorem{assumption}[assumption]{Assumption}
\aliascntresetthe{assumption}

\theoremstyle{definition}

\newaliascnt{definition}{theorem}
\newtheorem{definition}[definition]{Definition}
\aliascntresetthe{definition}

\newaliascnt{remark}{theorem}
\newtheorem{remark}[remark]{Remark}
\aliascntresetthe{remark}

\crefname{theorem}{theorem}{theorems}
\Crefname{theorem}{Theorem}{Theorems}
\crefname{lemma}{lemma}{lemmas}
\Crefname{lemma}{Lemma}{Lemmas}
\crefname{corollary}{corollary}{corollaries}
\Crefname{corollary}{Corollary}{Corollaries}
\crefname{proposition}{proposition}{propositions}
\Crefname{proposition}{Proposition}{Propositions}
\crefname{conjecture}{conjecture}{conjectures}
\Crefname{conjecture}{Conjecture}{Conjectures}
\crefname{assumption}{assumption}{assumptions}
\Crefname{assumption}{Assumption}{Assumptions}
\crefname{definition}{definition}{definitions}
\Crefname{definition}{Definition}{Definitions}
\crefname{remark}{remark}{remarks}
\Crefname{remark}{Remark}{Remarks}

\usepackage[dblblindworkshop, final]{neurips_2026}

\workshoptitle{AI for Accelerated Materials Design}

\usepackage[utf8]{inputenc} 
\usepackage[T1]{fontenc}    
\usepackage{hyperref}       
\usepackage{url}            
\usepackage{booktabs}       
\usepackage{amsfonts}       
\usepackage{nicefrac}       
\usepackage{microtype}      
\usepackage{xcolor}         

\title{On the Theoretical Limits of Learning Unconstrained Interatomic Potentials}

\author{%
  Prajwal Ananth\\
  Center for Applied Mathematics\\
  Cornell University\\
  Ithaca, NY
  \And
  Shuwen Yue\thanks{Corresponding author:
    \texttt{shuwen.yue@cornell.edu}.}\\
  R. F. Smith School of Chemical\\
  and Biomolecular Engineering\\
  Cornell University\\
  Ithaca, NY
}

\begin{document}

\maketitle

\begin{abstract}
Machine-learned interatomic potentials (MLIPs) accelerate molecular dynamics (MD) simulations by learning the energies and forces that govern atomic motion. Recent work has shown that relaxing physical inductive biases can improve computational efficiency and predictive accuracy, raising the question of whether energy-and-force matching alone is sufficient to ensure accurate sampling. We show that, for unconstrained hypothesis classes, the relationship between prediction error and sampling accuracy is exponentially ill-conditioned in dimension. In particular, prediction errors that decay exponentially with dimension can coexist with nearly maximal total variation distance between the learned and target distributions. We identify two failure mechanisms. The first is kinetic trapping, where relaxation to an accurate learned equilibrium requires exponentially long times, even when simulations are initialized from training configurations at which energies and forces are matched exactly. The second is equilibrium mismatch, where the learned stationary distribution itself remains far from the target. Under explicit assumptions, we show that equilibrium mismatch can persist even in local, symmetry-preserving models, where small per-atom errors coexist with a distributional discrepancy that approaches its maximum as system size increases. Finally, we establish global energy-control conditions that guarantee equilibrium accuracy and identify additional conditions required for rapid relaxation.
\end{abstract}

\section{Introduction}
\label{sec:introduction}
Machine-learned interatomic potentials (MLIPs) have emerged as a prominent application of scientific machine learning, enabling molecular simulations at a small fraction
of the cost of electronic structure methods. They are  trained to match energies and
forces on configurations drawn from expensive quantum mechanical calculations, and are subsequently deployed in classical molecular dynamics (MD) simulations to sample the equilibrium Boltzmann distribution associated with the underlying potential energy surface. The quantities of scientific interest, including free energies, transport
coefficients, and phase behavior, are functionals of this equilibrium distribution. Consequently, the practical utility of an MLIP depends not only on its accuracy in predicting energies and forces, but also on whether the resulting dynamics sample the correct distribution.

A growing body of empirical work reports that low test-set force and energy errors do not
translate into reliable simulations or accurate physical property predictions.
\citet{fu2023forces} find that force accuracy is poorly correlated with MD stability across
state-of-the-art models, while \citet{mannan2025uniffbench} report a similar ``reality gap''
with respect to experimental measurements. Similar failures have been observed for phonon spectra and
vibrational and thermal properties \citep{loew2025phonons}, particularly for
scalable direct-force models that do not enforce energy conservation
\citep{bigi2025darkside,fu2025smooth}.

These observations raise a natural question: \emph{Does minimizing the energy-and-force matching objective determine the distribution sampled by the induced molecular dynamics?} The question is further sharpened by a recent trend showing that relaxing physically motivated architectural constraints can improve efficiency and modeling capabilities. Whether optimizing an unconstrained learning objective can be sufficient at scale without additional architectural constraints or physical priors is one of the central open questions in the field \citep{creed2025}. Existing evidence, however, cannot
resolve it, for three reasons. First, it is empirical and tied to specific architectures,
datasets, and benchmark tasks, making the observed failures easy to attribute to insufficient data or limited model complexity. Second, it treats energy and force errors as primary evaluation metrics and reaches
mixed conclusions about their correlation with downstream performance. Third, it assesses
simulations through a small collection of low-dimensional observables, which can agree across
trajectories even when the underlying equilibrium distributions differ substantially. Consequently, it remains unclear whether unconstrained energy-and-force matching alone is sufficient to guarantee the correct equilibrium distribution, or whether additional constraints are required.

Our work addresses this gap by studying sampling guarantees and theoretically characterizing when energy-and-force matching ensures accurate sampling. Our main contributions are as follows:
\begin{itemize}[leftmargin=*, itemsep=2pt, topsep=2pt, parsep=0pt, partopsep=0pt]

\item \textbf{Exponential ill-conditioning for sampling accuracy:}
We show that the relationship between prediction error and sampling accuracy is exponentially ill-conditioned in dimension. In particular, arbitrarily small population energy and force errors can coexist with nearly maximal distributional error, establishing a fundamental limitation of unconstrained energy-and-force matching.

\item \textbf{High-dimensional failure mechanisms:}
We identify two sampling failure modes, kinetic trapping (\Cref{thm:kinetic-trapping,thm:data-interpolation}) and equilibrium mismatch (\Cref{thm:equilibrium-mismatch}), and prove that they can persist under exponentially small prediction errors, exact interpolation and initialization at training configurations.

\item \textbf{Roles and limitations of structural priors:}
We extend our analysis to conservative, local, permutation-invariant, and \(E(3)\)-symmetric hypothesis classes, characterizing the sampling guarantees and limitations of each structural prior (\Cref{tab:assumption-hierarchy}).

\item \textbf{Sufficient conditions for accurate sampling:}
We identify global energy-control conditions (\Cref{thm:sampling-guarantees}) that guarantee accurate equilibrium sampling and stronger conditions needed for rapid relaxation.

\end{itemize}

To our knowledge, this work provides the first formal analysis of the robustness of MLIP-driven MD simulations to energy-and-force matching errors. Our results are adversarial/worst-case in nature. They do not imply that any particular MLIP fails. Rather, they show that the learning objective does not preclude such failures.  Because these failures are established through explicit constructions, they reflect a limitation of energy-and-force matching rather than of any particular architecture or dataset. Our results therefore suggest that additional constraints beyond energy-and-force matching may be important for reliable MD sampling. While our work is primarily theoretical, we include a small set of controlled experiments on an Einstein crystal and FCC argon in \Cref{app:experimental-details} to illustrate the theory.

\section{Related work} Force matching is a standard route from reference calculations to
machine-learned potentials and underpins modern
MLIPs \citep{behler2021,zhang2018deepmd,schutt2021painn,batzner2022,batatia2022,musaelian2023,unke2021}. It is closely related to score matching because the score of a Boltzmann distribution is proportional to its force field. This connects our question to robustness analyses of score-based generative models, where a sharp contrast is known: diffusion models are robust to appropriately controlled $L^2$ errors in their time-dependent scores \citep{chen2023sampling,benton2024,li2024accel,liang2025}, whereas \citet{cao2026} recently showed that overdamped Langevin dynamics need not be robust to $L^2$ score error.

Our \Cref{thm:kinetic-trapping,thm:data-interpolation} extend their results (Theorems 1 and 7) to second-order (kinetic) Langevin dynamics used in molecular simulations. The proofs require a different, hypoelliptic exit-time argument because inertia fundamentally changes the short-time dynamics (see \Cref{rem:overdamped}). Our setting also requires control of energy errors and constrains the learned force field to be the negative gradient of a scalar potential. Because MLIPs are trained on reference energies and forces, we additionally establish control of the corresponding empirical errors.

Beyond the kinetic extension, our analysis establishes several additional results. \Cref{thm:equilibrium-mismatch} strengthens the asymptotic separation from $1-\eps$, for a constant $\eps$ fixed in advance, to $1-e^{-\Omega(d)}$, while preserving conservativity and a well-defined Gibbs equilibrium. \Cref{thm:sampling-guarantees} provides global energy-control conditions for equilibrium accuracy and additional conditions for rapid relaxation. Under explicit assumptions, \Cref{app:local-invariant-models} extends the equilibrium-mismatch result to a local, $E(3)$-invariant hypothesis class. The kinetic failures we construct are metastability phenomena \citep{bovier2004,bovier2005,menz2014}, induced by estimation error rather than inherited from the target potential.

\section{Background}
\subsection{Setup and notation}
\label{sec:setup-and-notation}
We consider a system of $N$ atoms with $d=3N$ configurational degrees of freedom. Positions and
velocities are $x,v\in\R^d$, and $\norm{\cdot}$ denotes the Euclidean norm. Throughout, we adopt
unit atomic masses. Heterogeneous masses can be absorbed by the coordinate rescaling
$x\mapsto M^{1/2}x$. We write $\Id$ for the identity
matrix, $A\succ0$ ($A\succeq0$) for positive (semi-)definiteness, $\lambda_{\min}(A)$ for the
smallest eigenvalue of a symmetric matrix, and $\N(m,\Sigma)$ for the Gaussian with mean $m$ and
covariance $\Sigma$.

The reference potential energy is $V:\R^d\to\R$ with force field $F=-\grad V$, and at
inverse temperature $\beta$ the target Boltzmann distribution is $\ptar\propto e^{-\beta V}$.
A machine-learned
interatomic potential (MLIP) is a surrogate $\Vhat_\theta$ with associated force field
$\Fhat_\theta=-\grad\Vhat_\theta$, obtained by minimizing the joint energy-and-force-matching
objective
\begin{equation}
\label{eq:energy-force-matching}
\theta \in \arg\min_\theta\ \E_{x\sim\ptar}
\Big[\,\abs{\Vhat_\theta(x)-V(x)}^2 + \lambda\,\norm{\Fhat_\theta(x)-F(x)}^2\,\Big],
\end{equation}
where $\lambda>0$ balances the two loss components. We suppress the subscript $\theta$ when the
parameterization is immaterial. \Cref{eq:energy-force-matching} is the ideal population objective with respect to the target distribution. Our analysis considers surrogates with arbitrarily small population loss, including constructions that exactly interpolate the empirical training data.

The Boltzmann distribution of the learned potential is $\pihat\propto e^{-\beta\Vhat}$, and we
write $\pihat_T$ for the position marginal of the distribution of $(x_t,v_t)$ at time $T$ under dynamics driven by $\Fhat$ from a specified initialization. Here $\pihat_T$
denotes the position distribution at time $T$, whereas $\pihat$ denotes the stationary
distribution induced by the learned potential. We measure sampling error by the total variation
distance between position marginals,
$\TV(\rho,\rho')=\sup_A\abs{\rho(A)-\rho'(A)}\in[0,1]$, and approximation error by the
$L^p$ norm,
$$\norm{\Fhat-F}_{L^p}=\big(\E_{x\sim\ptar}\norm{\Fhat(x)-F(x)}^p\big)^{1/p}$$ for $p\ge1$.
Note that the training objective \eqref{eq:energy-force-matching} controls the case $p=2$, while our results hold for every finite $p\ge1$.

\subsection{Ensemble sampling}
\label{sec:ensemble-sampling}
The thermodynamic state of an MD simulation is set by the ensemble it samples. We focus on the widely used canonical ($NVT$) ensemble, in which thermodynamic properties of interest are functionals of the equilibrium measure \citep{lelievre2016}. Many thermostat
algorithms exist \citep{leimkuhler2015}. We focus on the standard Langevin thermostat for clarity. The same mechanisms are expected to arise more broadly across ergodic $NVT$ and $NPT$ sampling schemes, although we do not formalize those extensions here. Under
Langevin thermostat and our unit-mass convention, the dynamics are governed by 
\begin{equation}
\label{eq:kinetic-langevin}
\dd x_t = v_t\,\dd t, \qquad
\dd v_t = F(x_t)\,\dd t - \gamma v_t\,\dd t + \sqrt{2\gamma\beta^{-1}}\,\dd B_t,
\end{equation}
where $B_t$ is a standard Brownian motion and
$\gamma>0$ is the friction coefficient.

Under mild regularity conditions, \eqref{eq:kinetic-langevin} is ergodic with unique stationary distribution given by the Gibbs measure on phase space, whose position marginal is $\ptar$ and whose velocity marginal is an independent Maxwell-Boltzmann distribution $\N(0,\beta^{-1}\Id)$. Consequently, long-time averages of observables converge to their canonical ensemble averages. In practice, however, simulations are  finite and the relevant questions are therefore whether the learned dynamics converge to the correct equilibrium and, if so, how rapidly they do so. These are quantified by the stationary distribution induced by the learned force field and its mixing time, defined as the time required for the distribution of $(x_t,v_t)$ to approach stationarity to a prescribed accuracy.

\subsection{Harmonic crystal potential}
\label{sec:harmonic-crystal}
To isolate failures arising from the learning objective rather than from the complexity of the underlying potential energy surface, we center our analysis on the harmonic crystal and then generalize our results to a broader class of target potentials.

The harmonic crystal is the simplest nontrivial model of a crystalline solid. It is smooth, strongly convex, and exactly solvable. It is obtained by expanding a smooth potential to second order about a stable equilibrium $\mu$:
\begin{equation}
\label{eq:harmonic-crystal}
V(x)=\tfrac12(x-\mu)^\top K(x-\mu), \qquad F(x)=-K(x-\mu),
\end{equation}
where $K=\grad^2V(\mu)\succ0$ is the force-constant matrix and its eigenpairs are the normal modes (phonons). The Boltzmann distribution is Gaussian, $\ptar=\N(\mu,(\beta K)^{-1})$, and is therefore strongly log-concave. Classical results imply that Langevin dynamics with arbitrary initialization converges to equilibrium in time polynomial in the dimension \citep{bakryemery1985,villani2009hypocoercivity,bakry2014analysis,bobkov1999isoperimetric}.

The isotropic case $K=k\Id$ is the Einstein model of a solid, in which every atom vibrates independently in an identical harmonic well \citep{einstein1907}. In this case, $\ptar$ factorizes into independent Gaussians $\bigotimes_i\N(\mu_i,(\beta k)^{-1})$. Harmonic systems are among the regimes in which learned force fields are generally expected to perform well. Observing sampling failures even in this setting suggests that similar failures may persist in more complicated potentials.

\section{Results}
We separate three questions that are often grouped under the term ``MD accuracy.'' First, does controlling population approximation error under the target distribution control the learned dynamics at a finite simulation time? Second, is the stationary distribution of the learned dynamics close to the target? Third, which architectural constraints and inductive biases are necessary to control either of these quantities and what failure modes they eliminate?

\begin{table}[!b]
\centering
\small
\caption{Error-to-sampling structure of the main results. ``Population error'' is measured with respect to the target distribution. The table distinguishes predictive accuracy, finite-time dynamics, and stationary-distribution accuracy.}
\label{tab:result-roadmap}
\begin{tabularx}{\textwidth}{@{}p{0.1\textwidth}YYY@{}}
\toprule
\textbf{Result} & \textbf{Setting} & \textbf{Error / Control condition} & \textbf{Physical conclusion} \\
\midrule
\Cref{thm:kinetic-trapping} & Harmonic target; 

displaced initialization & Exponentially small population error & Kinetic trapping for $e^{\Omega(d)}$ time \\
\Cref{thm:data-interpolation} & Harmonic target; training-data initialization & Exponentially small population error + Exact interpolation  & Kinetic trapping for $e^{\Omega(d)}$ time \\
\Cref{thm:equilibrium-mismatch} & Broader class of confining targets; arbitrary initialization & Exponentially small population error & Incorrect stationary distribution  at distance $1-e^{-\Omega(d)}$ \\
\Cref{thm:sampling-guarantees} & General targets with global energy control & Uniform or one-sided energy bounds & Stationary accuracy; uniform control also ensures relaxation. \\
\Cref{thm:local-equilibrium-mismatch} & Local, invariant site-energy models & Per-atom error $e^{-\Omega(\nu)}$ & Global equilibrium error $1-e^{-\Omega(N)}$ \\
\bottomrule
\end{tabularx}
\vspace{-2mm}
\end{table}

We answer the first two questions negatively under unconstrained energy-and-force matching and then study how additional structure changes these conclusions. \Cref{sec:non-equilibrium-initialization,sec:data-based-initialization} establish kinetic trapping, including under exact interpolation and initialization from the training set. \Cref{sec:general-targets} establishes equilibrium mismatch for a broader class of targets and arbitrary initialization. \Cref{app:local-class-failure} shows that locality and symmetry can potentially reduce the scale of hidden errors but do not prevent their amplification with system size. Finally, \Cref{sec:sufficient-conditions} separates conditions for accurate stationary distributions from those required for rapid relaxation.

\subsection{Kinetic trapping under exponentially small energy and force error}
\label{sec:non-equilibrium-initialization}
Our first result provides a simplified setting for \cref{thm:data-interpolation} and isolates finite-time dynamical failure. It establishes a lower bound on the total-variation distance between the distribution generated from a displaced initialization and the equilibrium distribution of an Einstein crystal.

\begin{theorem}\label{thm:kinetic-trapping}
Fix a constant $\xi>5$. Let $V(x)=\tfrac12\norm{x-\mu}^2$ with $\norm{\mu}=\xi\sqrt d$, so that $\ptar=\N(\mu,\Id)$. There
exists $\Vhat\in C^\infty(\R^d)$, with $\Fhat:=-\grad\Vhat$, and $\sup_x\norm{\grad^2\Vhat(x)}_{\rm op}\le H$
for a $d$-independent constant $H$, such that:
\begin{enumerate}[label=\emph{(\roman*)}, leftmargin=*] 
\item For every $p\ge1$, $\exists\, d_0(p,\xi)$ such that for $d\ge d_0(p,\xi)$,
\[
\E_{x\sim\ptar}\big[\abs{\Vhat(x)-V(x)}^p\big]^{1/p}\le e^{-\Omega(d)},\qquad
\E_{x\sim\ptar}\big[\norm{\Fhat(x)-F(x)}^p\big]^{1/p}\le e^{-\Omega(d)}
\]

\item There are constants $c,d_1>0$, depending only on $\xi$, such that for all $d\ge d_1$ and all $T\le e^{cd}$,
the distribution $\pihat_T$ induced by Langevin dynamics with $\Fhat$, initialized at
$x_0,v_0\stackrel{\text{iid}}{\sim}\N(0,\Id)$, satisfies
\[
\TV\big(\pihat_T,\ptar\big)\ge1-e^{-\Omega(d)}
\]
\end{enumerate}
Moreover $\Fhat$ is Lipschitz with a constant independent of $d$, so
\eqref{eq:kinetic-langevin} has a unique strong solution.
\end{theorem}

\Cref{thm:kinetic-trapping} shows that, in high dimensions, the distribution induced by MD can remain far from $\ptar$ over every polynomial time horizon despite exponentially small energy and force errors. The construction (\Cref{app:kinetic-trapping-proof}) introduces a smooth basin far from $\mu$, where $\ptar$ assigns mass $e^{-\Omega(d)}$, and interpolates back to the target potential wherever $\ptar$ places appreciable mass. The perturbation is therefore invisible to the population matching objective while substantially altering the resulting dynamics. Importantly, $\Vhat$ is smooth and conservative, its force is globally Lipschitz, and its Hessian bound is independent of dimension. These properties ensure regularity and well-posed dynamics, but the theorem shows that they do not control the global wells and barriers that determine sampling. In particular, such hidden basins can coexist with force errors that decay exponentially with dimension.

The freedom exploited by the construction reflects a practical feature of MLIP training. Reference data can be generated only for configurations on which the electronic structure calculation remains well defined and computationally tractable. Over much of configuration space such calculations are ill-defined: strongly overlapping atoms and dissociated fragments prevent SCF convergence or demand treatments beyond standard single-reference methods. The Boltzmann density is strictly positive throughout, so these regions may be absent from the data while still having nonzero probability under the target distribution.

It is tempting to dismiss such regions as dynamically irrelevant. For a diatomic, this intuition is
correct: the trajectory is confined by a dissociation barrier far exceeding $k_BT$, and no
reasonable initialization escapes it. Two effects break this intuition in high dimensions.
First, that barrier is a property of $V$, whereas the simulation is integrated on $\Vhat$, which
\eqref{eq:energy-force-matching} constrains only where $\ptar$ places mass. Second, independently of reachability, low $\ptar$ density does not imply low $\pihat$-mass. For a target concentrating on a shell of radius $\Theta(\sqrt d)$, the region beyond radius $R$ may have $\ptar$-mass $e^{-\Omega(d)}$ while its volume scales as $R^{\Theta(d)}$. Suppressing $\pihat$-mass in such regions therefore requires an energy penalty that grows with $d$, which an $L^2$ bound does not guarantee. This energy-entropy competition is what \Cref{thm:sampling-guarantees} rules out, through a one-sided lower bound $\Vhat\ge V-\delta$ that holds
globally rather than only on the equilibrium region.

\textbf{Proof Sketch:}
The proof is based on a hypoelliptic exit-time bound. The technical core is \Cref{lem:kinetic-exit-time} in \Cref{app:kinetic-exit}. Consider the linear kinetic Langevin dynamics
\[\dd x_t=v_t\,\dd t,\qquad \dd v_t=(-\alpha x_t-\gamma v_t)\,\dd t+\sqrt{2\gamma}\,\dd B_t,\]
initialized with $\|x_0\|\le\rho_x\sqrt{d}$ and $\|v_0\|\le\rho_v\sqrt{d}$. The lemma shows that the exit time from the ball $\{\|x\|\le r\sqrt{d}\}$ exceeds $e^{c_\star d}$ with probability at least $1 - Ce^{-c_\star d}$, provided
\[\alpha\ge\alpha_\star=\Theta\!\left(\frac{1+\gamma^2\rho_x^2+\rho_v^2}{(r-\rho_x)^2}\right).\]
A key feature is that the effect of the initial velocity enters only through $\rho_v/\sqrt{\alpha}$, so increasing the stiffness suppresses ballistic excursions. Moreover, although $\alpha_\star$ grows as the margin $r-\rho_x$ shrinks, the exponent $c_\star=\alpha_\star(r-\rho_x)^2/64$ remains unchanged. This scale invariance allows the same lemma to be applied at the much smaller radii required for \Cref{thm:data-interpolation}.

Since $\|x_0\|,\|v_0\|\le1.1\sqrt{d}$ with probability $1-e^{-\Omega(d)}$, taking $r=4$ and $\rho_x=\rho_v=1.1$ confines the trajectory to $\|x_t\|\le4\sqrt{d}$ for every $T\le e^{c_\star d}$. In contrast, $\ptar$ concentrates outside this ball because $\|\mu\|=\xi\sqrt{d}$ with $\xi>5$. Consequently, the distribution induced by the dynamics remains at total variation distance $1-e^{-\Omega(d)}$ from $\ptar$ throughout this time horizon. Meanwhile, the energy and force errors remain exponentially small because $\Vhat$ differs from $V$ only inside $\|x\|\le b_\xi\sqrt{d}$, where $b_\xi<\xi-1$, a region whose $\ptar$ mass is $e^{-\Omega(d)}$. The complete proof is in \Cref{app:kinetic-trapping-proof}.

\subsection{Kinetic Trapping with exact interpolation and initialization from training data}
\label{sec:data-based-initialization}

A standard practice is to initialize MD from a configuration within the training set. We show that even when the same samples are used both to fit an exactly interpolating $\Fhat$ and to initialize MD, the failure mode of \Cref{thm:kinetic-trapping} persists.

\begin{figure}[!t]
    \centering
    \begin{subfigure}[t]{0.48\textwidth}
        \centering
        \includegraphics[width=0.9\linewidth]{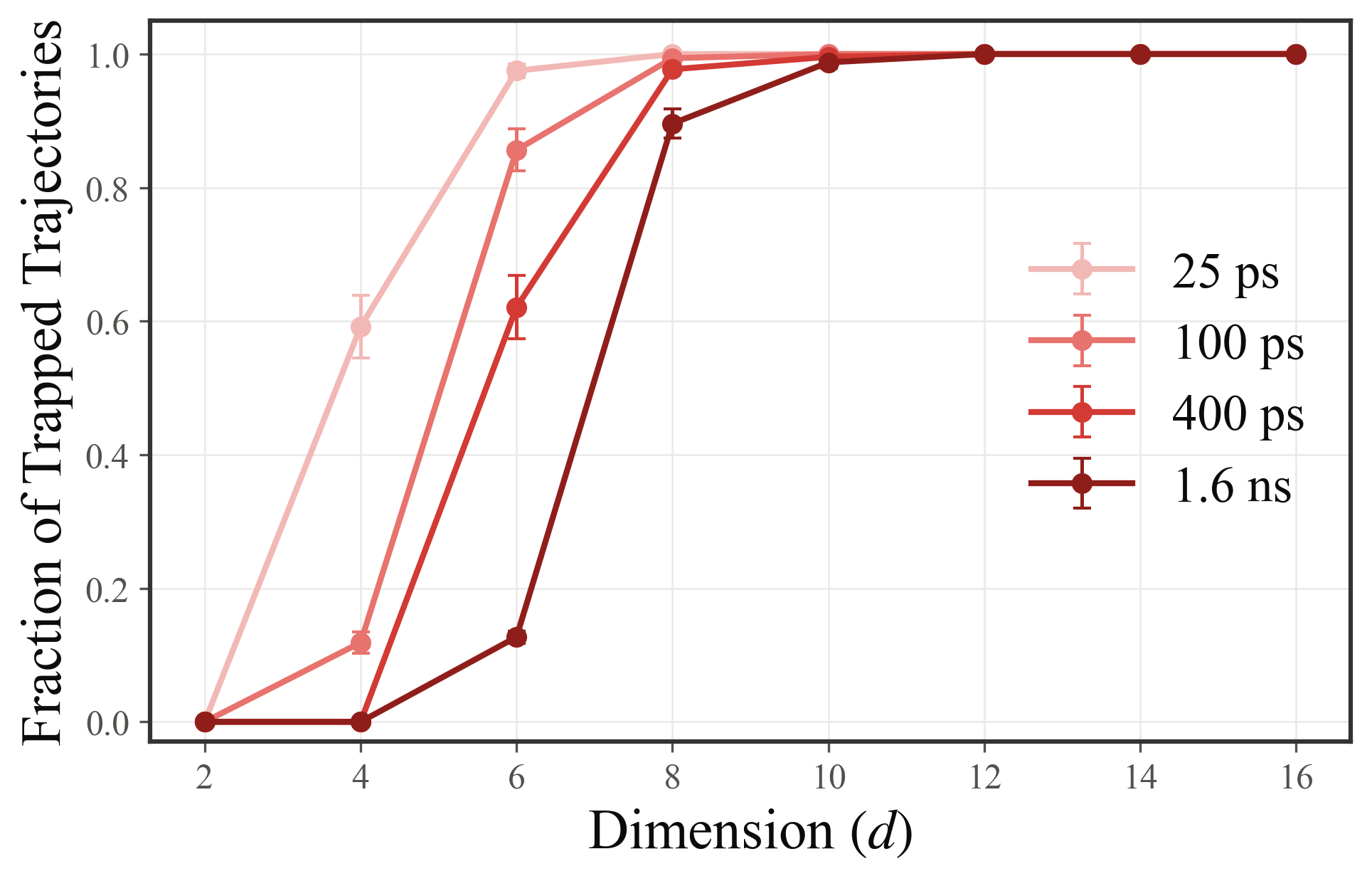}
        \caption{Kinetic trapping in the Einstein crystal}
        \label{fig:einstein-escape-times}
    \end{subfigure}
    \hfill
    \begin{subfigure}[t]{0.48\textwidth}
        \centering
        \includegraphics[width=0.9\linewidth]{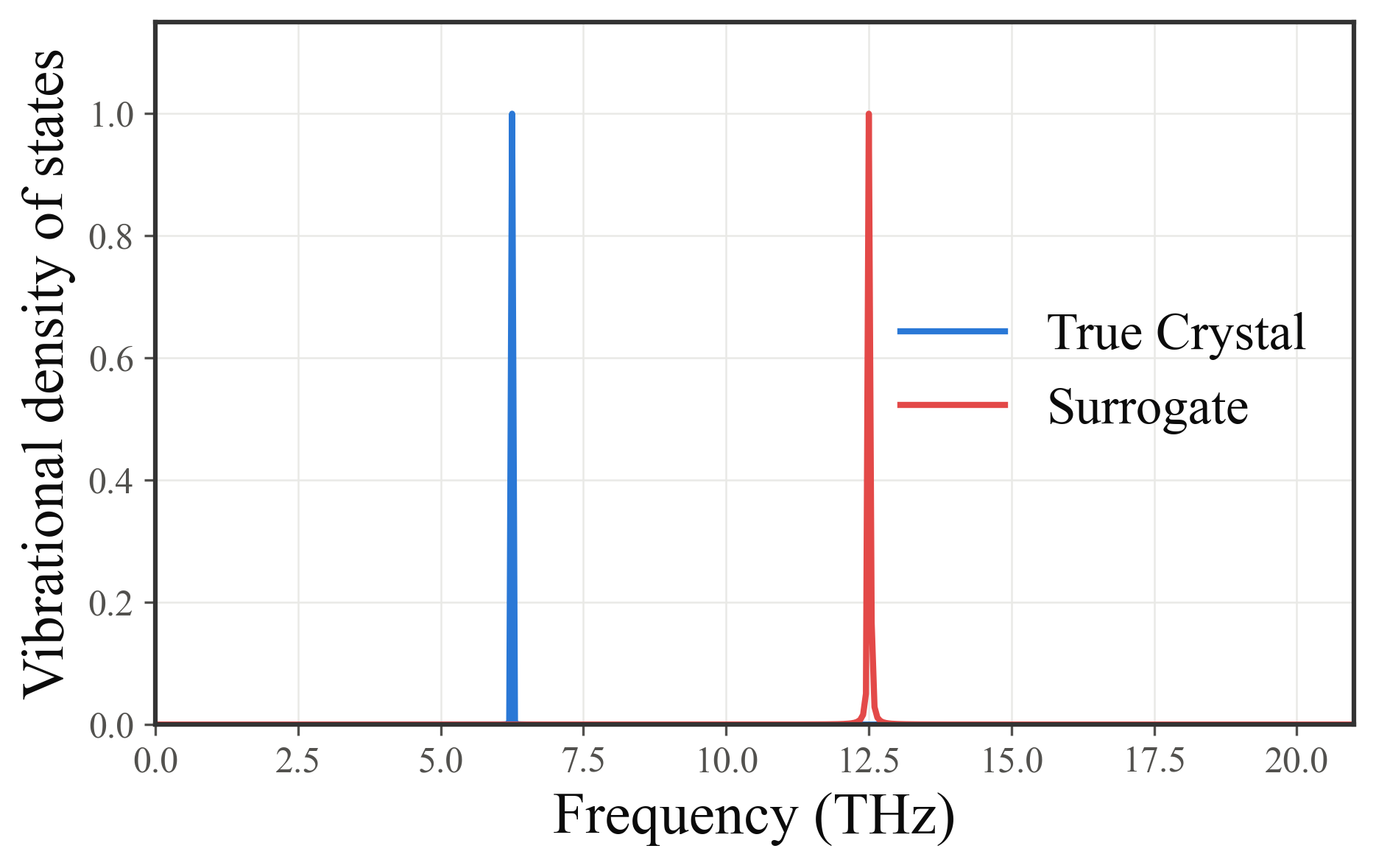}
        \caption{Phonon vibrational density of states}
        \label{fig:einstein-phonon-density}
    \end{subfigure}
    \caption{Illustration of kinetic trapping mechanism underlying \Cref{thm:kinetic-trapping,thm:data-interpolation}  (details in \Cref{app:einstein-crystal-experiment}). (a) Escape time grows as $e^{\Omega(d)}$ with an $\mathcal{O}(1)$ rate per degree of freedom. The fully trapped plateau extends to arbitrarily large $d$. (b) The true Einstein crystal shows a single sharp line at the analytic $\nu_E=6.25\,\mathrm{THz}$; the trapped surrogate's line is at $12.5\,\mathrm{THz}$.}
    \label{fig:einstein-crystal-mechanism}
    \vspace{-2mm}
\end{figure}
\begin{theorem}
\label{thm:data-interpolation}
Let $\mu\in\R^d$, $V(x)=\tfrac12\norm{x-\mu}^2$, so that $\ptar=\N(\mu,\Id)$, and let
$x_1,\dots,x_n$ with $n=\mathrm{poly}(d)$ be i.i.d.\ samples from $\ptar$. With probability
$1-e^{-\Omega(d)}$ over the draw of the training set, there exists $\Vhat\in C^\infty(\R^d)$,
depending on $x_1,\dots,x_n$, with $\Fhat:=-\grad\Vhat$ and
$\sup_x\norm{\grad^2\Vhat(x)}_{\rm op}\le H$ for a $d$-independent constant $H$, such that:
\begin{enumerate}[label=\emph{(\roman*)}, itemsep=1pt,topsep=1pt, leftmargin=*]
\item $\Vhat(x_i)=V(x_i)$ and $\Fhat(x_i)=F(x_i)$ for every $i=1,\dots,n$; that is, $\Vhat$
attains zero empirical energy-and-force matching loss on the training set.
\item For every $p\ge1$, $\exists\, d_0(p)$ such that for $d\ge d_0(p)$,
\[
\E_{x\sim\ptar}\big[\abs{\Vhat(x)-V(x)}^p\big]^{1/p}\le e^{-\Omega(d)},\qquad
\E_{x\sim\ptar}\big[\norm{\Fhat(x)-F(x)}^p\big]^{1/p}\le e^{-\Omega(d)}
\]
\item There are constants $c,d_1>0$ such that for all $d\ge d_1$ and all $T\le e^{cd}$,
the distribution $\pihat_T$ induced by Langevin dynamics with $\Fhat$, initialized at
$x_0\sim\frac1n\sum_i\delta_{x_i}$ and $v_0\sim\N(0,\Id)$, satisfies
\[
\TV\big(\pihat_T,\ptar\big)\ge1-e^{-\Omega(d)}
\]
\end{enumerate}
\end{theorem}

As before, the uniform Hessian bound makes $\Fhat$ globally Lipschitz, so \eqref{eq:kinetic-langevin} has a unique strong solution. The full proof is given in \Cref{app:data-initialization-proof}.

\Cref{thm:data-interpolation} strengthens \Cref{thm:kinetic-trapping} in two ways: the initialization is a training
configuration drawn from $\ptar$, and the learned potential fits the training data exactly in both energy and force. The
failure therefore cannot be detected by the standard energy and force metrics, with an empirical loss driven
to zero and an exponentially small test error coexisting with dynamics that remains far from the target over
any $\mathrm{poly}(d)$ horizon. The proof exploits the fact that in high dimensions each training
configuration occupies a neighborhood of exponentially small $\ptar$-measure, so the surrogate can
place a basin around every sample while still matching energies and forces exactly there.
The $\mathrm{poly}(d)$ basins together carry probability $e^{-\Omega(d)}$, leaving the population
errors exponentially small while trajectories launched from the data remain confined. 

\subsection{Equilibrium mismatch for arbitrary initialization}
\label{sec:general-targets}

\Cref{thm:kinetic-trapping,thm:data-interpolation} establish finite-horizon failures. One might therefore suspect that they arise from initialization or a limited simulation budget, and that alternative initializations or longer trajectories would recover $\ptar$. This is indeed true for the preceding constructions. By \Cref{prop:one-sided-equilibrium-control}, both learned potentials induce accurate stationary distributions, so the failure is purely kinetic. We next establish a stronger failure mode in which the learned stationary distribution is itself incorrect. This failure applies to a broader class of target distributions, holds for arbitrary initialization, and persists as $t\to\infty$. Throughout this section, we assume that the target satisfies \Cref{ass:confinement,ass:target-concentration} and take $\inf V=0$ without loss of generality.

\Cref{ass:confinement} is the standard Lyapunov condition under which \eqref{eq:kinetic-langevin} is
hypoelliptic and geometrically ergodic with the phase-space Gibbs measure as its unique invariant
measure \citep{hormander1967,mattingly2002}. \Cref{ass:target-concentration} says the target concentrates on a ball
of the physical radius $O(\sqrt d)$, true for every strongly log-concave target and more generally
under a dimension-free log-Sobolev inequality, and that the energy is extensive on that scale, true
whenever $\grad^2V$ is bounded.

\begin{theorem}
\label{thm:equilibrium-mismatch}
Let $\ptar\propto e^{-\beta V}$ satisfy \Cref{ass:confinement,ass:target-concentration} and fix $\kappa_2>0$. There
exists $\Vhat\in C^\infty(\R^d)$, with $\Fhat:=-\grad\Vhat$ and
$\sup_x\norm{\grad^2\Vhat(x)}_{\rm op}\le L+H'$ for a $d$-independent constant $H'$, such that:
\begin{enumerate}[label=\emph{(\roman*)}, itemsep=1pt,topsep=1pt, leftmargin=*]
\item For every $p\ge1$, $\exists\,d_0(p)$ such that for $d\ge d_0(p)$,
\[
\E_{x\sim\ptar}\big[\abs{\Vhat(x)-V(x)}^p\big]^{1/p}\le e^{-\Omega(d)},\qquad
\E_{x\sim\ptar}\big[\norm{\Fhat(x)-F(x)}^p\big]^{1/p}\le e^{-\Omega(d)}
\]
\item Langevin dynamics \eqref{eq:kinetic-langevin} driven by $\Fhat$ is geometrically ergodic, its
unique stationary distribution is the phase-space Gibbs measure of $\Vhat$, whose position marginal
is $\pihat\propto e^{-\beta\Vhat}$, and for all $d$ larger than a constant,
\[
\TV\big(\pihat,\ptar\big)\ \ge\ 1-e^{-\Omega(d)}\]
\item Consequently, for \emph{every} initial law with
$\E\big[V(x_0)+\norm{x_0}^2+\norm{v_0}^2\big]<\infty$,
\[
\lim_{t\to\infty}\TV\big(\pihat_t,\ptar\big)\ =\ \TV\big(\pihat,\ptar\big)\ \ge\ 1-e^{-\Omega(d)}
\]
\end{enumerate}
Moreover $\Vhat=V$ outside a ball of radius $2\theta\sqrt d$ carrying $\ptar$-mass $e^{-\Omega(d)}$,
and $\norm{\Vhat-V}_\infty=O(d/\beta)$.
\end{theorem}

The construction (\Cref{app:equilibrium-mismatch-proof}) subtracts from $V$ a smooth well of depth $\Delta$, supported on a ball of radius $2\theta\sqrt d$ placed just beyond the region where $\ptar$ concentrates. Two competing requirements determine $\Delta$. The well must be deep enough that $e^{\beta\Delta}$ compensates for the exponentially small prior mass that $\pihat$ would otherwise assign to the ball, requiring $\beta\Delta=\Omega(d)$. At the same time, it must remain shallow and wide enough that the induced curvature, $\Delta/(\theta\sqrt d)^2$, stays independent of $d$, which is achieved by taking $\beta\Delta=O(d)$ and a width of $\Theta(\sqrt d)$. These requirements are compatible, yielding a confining potential with a well-defined Boltzmann equilibrium that is not the target equilibrium.

\begin{remark}
\label{rem:per-atom-scaling}
The construction is extensive rather than divergent. The energy offset is $\Delta=O(d/\beta)$ in total, which is $O(k_BT)$ per atom, and the force perturbation is $O(\sqrt d)$ in norm, which is $O(1)$ per atom. Consequently, a model that is accurate to $e^{-\Omega(d)}$ on the sampled basin but underestimates the energy by only a few $k_BT$ per atom in an unsampled region can nevertheless place all equilibrium mass essentially there.
\end{remark}
\begin{remark}
\label{rem}
The construction is conservative, and by the confinement properties established in \Cref{thm:equilibrium-mismatch}, the learned dynamics has a well-defined Gibbs equilibrium. Non-conservative direct-force models can exhibit additional failures, including rotational drift and non-gradient escape mechanisms. For example, the overdamped construction of \citet{cao2026} introduces a constant drift within a narrow cone that transports the process to infinity. Because this field is not a gradient, that construction does not apply to conservative models. These mechanisms may help explain empirical failures of non-conservative models recently observed by \citet{bigi2025darkside} and \citet{fu2025smooth}. However, \Cref{thm:kinetic-trapping,thm:equilibrium-mismatch} show that conservative models with well-defined equilibria need not have the correct equilibrium distribution or mix efficiently.
\end{remark}

All of our main results are stated for smooth, conservative potentials without architectural constraints such as locality or symmetry, thereby emphasizing the need for additional structure. Standard MLIPs impose additional structure through finite-range locality, permutation invariance, and $E(3)$-invariant energies with equivariant forces. It remains unclear whether these inductive biases are sufficient to rule out the failure mechanisms identified here, whether further constraints are required, and how sampling guarantees change under these architectural constraints.

\Cref{app:local-invariant-models} shows that the equilibrium-mismatch mechanism of \Cref{thm:equilibrium-mismatch} can persist under these inductive biases when a competing structure is locally recognizable, meaning that its local atomic environments can be distinguished from those typical of the target distribution within the model's finite receptive field. Under explicit concentration and bounded-overlap assumptions, we construct a smooth, finite-cutoff, size-extensive site-energy model with a \(2L\)-hop receptive field whose per-atom energy and force errors are $e^{-\Omega(\nu)}$, while its stationary distribution satisfies $\TV(\pihat,\ptar)\ge 1-e^{-\Omega(d)}$.
Here, $\nu\asymp z^{L-1}$ in the bounded-degree regime, where $z$ is the local coordination number. Compared with the unconstrained construction of \Cref{thm:equilibrium-mismatch}, locality could potentially raise the error that can be hidden from $e^{-\Omega(d)}$ to $e^{-\Omega(\nu)}$. However, it does not prevent an undetected local bias from accumulating extensively when a competing structure, such as a distorted lattice or competing polymorph, is locally recognizable and differs from the target by only $O(1)$ free energy per atom.

\begin{table}[!b]
\vspace{-2mm}
\centering
\small
\caption{Summary of guarantees provided by common assumptions and the aspects of sampling they leave uncontrolled.}
\label{tab:assumption-hierarchy}
\begin{tabularx}{\textwidth}{@{}p{0.27\textwidth}YY@{}}
\toprule
\textbf{Assumption} & \textbf{Provides} & \textbf{Does not control} \\
\midrule
Small $L^p$ energy/force error
& Accuracy on typical configurations
& Existence of equilibrium \\

Conservative force + Confinement
& Scalar potential and Gibbs equilibrium
& Equilibrium accuracy and mixing time \\

Smoothness + bounded Hessian
& Lipschitz force and well-posed dynamics
& Global wells, barriers, sampling \\

Locality, permutation and $E(3)$ symmetry
& Exact transformation laws, extensivity and a potential  receptive-field safeguard
& Accumulation of residual local bias with system size \\

One-sided global control + on-target accuracy
& Equilibrium accuracy
& High barriers and kinetic trapping \\

Uniform two-sided control + functional inequality
& Equilibrium accuracy and exponential relaxation
& A strong global condition not supplied by ordinary matching \\
\bottomrule
\end{tabularx}
\end{table}

Because $\Vhat$ is expressed as a sum of atomic energies, the trained model can be evaluated at system sizes different from those used during training, as is common practice in MLIP applications. As shown in \Cref{app:local-invariant-models}, the per-atom error bounds remain independent of system size, whereas the equilibrium discrepancy approaches one exponentially fast in $N$. Consequently, the model can exhibit unchanged per-atom validation error as system size increases while its equilibrium distribution becomes increasingly inaccurate. This separation highlights the need for caution when deploying models trained on smaller systems for inference and sampling in larger systems (\Cref{rem:system-size-transfer}).

\subsection{Sufficient Conditions}
\label{sec:sufficient-conditions}
The negative results reveal two distinct requirements for accurate sampling. Thermodynamic accuracy requires the learned invariant distribution to be close to the target, while dynamical accuracy additionally requires that this distribution be reached on the simulation timescale. The following theorem gives sufficient conditions for each.
\begin{theorem}
\label{thm:sampling-guarantees}
Let $\pihat\propto e^{-\beta\Vhat}$ and $\ptar\propto e^{-\beta V}$.
\begin{enumerate}[leftmargin=*,itemsep=1pt,topsep=1pt,label=\emph{(\roman*)}]
\item If $\norm{\Vhat-V}_\infty\le\delta$, then
$\tfrac{\dd\pihat}{\dd\ptar}\in[e^{-2\beta\delta},e^{2\beta\delta}]$ and
$\TV(\pihat,\ptar)\le1-e^{-2\beta\delta}\le2\beta\delta$.
\item Suppose (a) $\Vhat\ge V-\delta$
everywhere, and (b) there is a set $S$ with $\ptar(S)\ge1-\eta$ and $\Vhat\le V+\delta'$
on $S$. Then 
$\tfrac{\dd\pihat}{\dd\ptar}\le\tfrac{e^{\beta(\delta+\delta')}}{1-\eta}
\quad \text{and} \quad \TV(\pihat,\ptar)\le 1-e^{-\beta(\delta+\delta')}+\eta$.
\item Suppose in addition that $\ptar$ satisfies a log-Sobolev inequality with
constant $\rho>0$ and $\sup_x\norm{\grad^2\Vhat(x)}_{\rm op}\le H$. Under (i), $\pihat$ satisfies a log-Sobolev inequality with constant at least
$\rho\,e^{-2\beta\delta}$. Consequently, Langevin dynamics driven by $\Fhat = -\nabla\Vhat$ converges exponentially to its invariant distribution $\pihat$, with a hypocoercive rate depending on $\rho e^{-2\beta\delta}$, $H$, $\gamma$, and $\beta$. Under condition (ii) alone, no analogous uniform relaxation guarantee follows.
\end{enumerate}
\end{theorem}

Part (i) imposes a uniform error bound and thereby rules out both failure modes. Such a guarantee is stronger than what is typically obtained from unconstrained matching: an $L^p$ bound under $\ptar$ need not provide meaningful control in regions where $\ptar$ places very little mass, including regions that may appear dynamically irrelevant under the target distribution. Part (ii) weakens this requirement in the direction relevant for practice, requiring accuracy only on the support of the target and replacing global accuracy with a global lower bound $\Vhat \ge V-\delta$. This condition rules out the mechanism in \Cref{thm:equilibrium-mismatch}, which constructs a well of depth $\Delta=\Omega(d/\beta)$ below the true potential: preventing wells deeper than $\delta$ limits the equilibrium mass that can accumulate off-support. However, it provides no control over barriers and therefore no guarantee on relaxation.

Part (ii) can provide a design target. Enforcing a lower bound on the predicted energy, for example through a repulsive baseline or a coercive residual parameterization $\widehat V=V_0+f_\theta$, rules out wrong-equilibrium failures without sacrificing on-support accuracy. It does not require extrapolation, but only that $\widehat V$ does not fall arbitrarily below $V$ in regions where the target places little mass.

In practice, constructing such coercive baselines with limited knowledge of the true potential energy surface from classical physics remains an important direction for future work. Physical priors, such as short-range ZBL repulsion, can provide local protection against unphysical energy collapse when atoms approach very small separations and have been incorporated into recent MLIPs \citep{tompa2026finetuning,zhang2026constructing} for this purpose. However, such priors address only a specific class of off-support failures and do not provide the global one-sided energy control required by our result. In particular, the failure mechanisms constructed here need not involve atomic collisions or anomalously short interatomic distances, and can therefore remain invisible to purely short-range repulsive corrections. Developing broadly applicable priors that provide the required energy control without imposing detailed knowledge of the true potential  remains an open problem.

\section{Conclusion}
\label{sec:conclusion}
We have shown that the relationship between energy-and-force matching error and sampling accuracy is exponentially ill-conditioned in dimension. Smooth, conservative surrogates with dimension-independent Hessian bounds can interpolate the training data exactly, achieve exponentially small population energy and force errors, and still induce dynamics whose distribution remains far from the target, either because the target equilibrium cannot be reached in sub-exponential time (kinetic trapping) or because the learned equilibrium distribution itself is incorrect (equilibrium mismatch). Equilibrium mismatch can persist under locality, permutation invariance, and $E(3)$-equivariance, although the local construction may incur a receptive-field-dependent approximation error. A one-sided global lower bound on the learned energy rules out the equilibrium-mismatch mechanism, while rapid relaxation additionally requires control of barriers and global geometry. Together, our impossibility results for unconstrained MLIPs provide theoretical support for imposing structure beyond energy-and-force matching and identify design targets for such constraints. Our results are existential and worst case and therefore do not establish how frequently these failure modes arise in practice. Characterizing their prevalence under realistic training procedures and data distributions is an important direction for future work.

\section*{Acknowledgments}
We acknowledge startup funding from Cornell University and computing resources provided by the Cornell University Center for Advanced Computing (CAC). 

\medskip
{
\small
}
\bibliographystyle{plainnat}

\bibliography{references}
\appendix

\crefalias{section}{appendix}
\numberwithin{equation}{section}

\section{Experiments}
\label{sec:experiments}
\label{app:experimental-details}

We include a small set of controlled simulations to illustrate the sampling failures predicted by our theory.  These simulations are not intended to reproduce realistic MLIP training pipelines or support claims about practical model performance. Their purpose is to provide numerical demonstrations of the limitations of unconstrained energy-and-force matching. Although our
theory concerns worst-case constructions, simulation collapse has been empirically observed in state-of-the-art
MLIPs and reported in the literature \citep{fu2023forces, mannan2025uniffbench}, suggesting that related failure mechanisms may arise
in practice (see \Cref{fig:forces}). In realistic MLIP training pipelines, however, isolating and verifying the precise mechanism
responsible for a given simulation failure is generally impractical, since multiple sources of model and sampling
error may interact.

We first construct an Einstein-crystal model that directly exhibits the kinetic-trapping mechanism of \Cref{thm:kinetic-trapping,thm:data-interpolation}. We then study a harmonic argon crystal and show that neural networks trained by energy-and-force matching can produce inaccurate molecular dynamics despite small held-out force errors. These experiments are intended as finite-time illustrations of the phenomena captured by the theory rather than numerical realizations of the theorem constructions.

\subsection{Einstein crystal: explicit construction}
\label{app:einstein-crystal-experiment}

We begin with an Einstein crystal for which both the equilibrium distribution and the relevant thermodynamic observables are available analytically. We use an fcc Al lattice with lattice constant $a_0=4.05$ \AA{} and a $2\times2\times2$ supercell containing $N=32$ atoms ($d=3N=96$). Each atom is tethered to its lattice site $s_i$ by an isotropic harmonic potential with force-constant matrix $K=kI$, where $k=m\omega_E^2$. We set the Einstein temperature to
$\theta_E=h\nu_E/k_B=300$ K and the atomic mass to $m=26.98$ u, which gives
$k\simeq4.3$ eV/\AA$^2$ and $\nu_E=\omega_E/(2\pi)=6.25$ THz. The corresponding classical predictions are
$\langle u^2\rangle=3k_BT/k=0.0180$ \AA$^2$, $C_v=3k_B$ per atom, and a single vibrational-density-of-states (VDOS) peak at $\nu_E$.

To illustrate kinetic trapping, we modify the surrogate potential by introducing an additional basin with local stiffness $\alpha=4k$. The basin is centered at $\mu$, with $\lVert\mu-s\rVert=6$ \AA{}, and is gated by a $C^1$ smoothstep that is equal to one within $r_a=2.0$ \AA{} and vanishes beyond $r_c=4.0$ \AA{}. A trajectory initialized inside this basin remains trapped there, reproducing the mechanism underlying \Cref{thm:kinetic-trapping,thm:data-interpolation}.

We generate Langevin trajectories with ASE at $T=300$ K. Initial velocities are drawn from the Maxwell--Boltzmann distribution with center-of-mass motion removed. We use a $3$ fs timestep, $10$ ps of equilibration, and $40$ ps of NVT sampling with friction $0.02$ fs$^{-1}$. The VDOS is estimated from the velocity power spectrum of a subsequent $20$ ps NVE trajectory integrated with velocity-Verlet using a $2$ fs timestep.

We report the mean-square displacement
\(
\langle u^2\rangle
=
\langle\lVert r_i-\bar r_i\rVert^2\rangle_{i,t},
\)
the structural error
\(
\left[
N^{-1}\sum_i\lVert\bar r_i-s_i\rVert^2
\right]^{1/2},
\)
the heat capacity
\(
C_v=
\frac{\mathrm{Var}(E_{\mathrm{tot}})}
{k_BT^2N},
\)
and the VDOS,
\[\mathrm{VDOS}(\nu)
\propto
\sum_i|\widehat{\mathbf v}_i(\nu)|^2.
\]
True MD reproduces the analytic predictions, whereas the trapped surrogate gives an incorrect mean-square displacement, mean structure, and phonon frequency. The heat capacity remains close to $3k_B$ per atom because the classical harmonic heat capacity does not depend on the spring stiffness. The numerical comparison is given in \Cref{tab:einstein-observables}.

\begin{table}[!t]
\centering
\small
\caption{Analytic predictions and molecular-dynamics estimates for the Einstein-crystal construction. The surrogate trajectory is kinetically trapped (\Cref{fig:einstein-crystal-mechanism}). Its heat capacity remains apparently correct even though its displacement statistics and mean structure are incorrect.}
\label{tab:einstein-observables}
\begin{tabular}{@{}lccc@{}}
\toprule
Physical property & Analytic & True MD & Surrogate MD \\
\midrule
Debye--Waller $\langle u^2\rangle$ (\AA$^2$) & 0.0180 & 0.0180 & \textbf{0.0045} \\
Structure error (\AA) & 0 & 0.004 & \textbf{1.06} \\
Heat capacity $C_v$ ($k_B$/atom) & 3.00 & 3.03 & \textbf{3.03} \\
\bottomrule
\end{tabular}
\end{table}

We also examine how the trapping behavior changes with dimension. We consider
$d\in\{2,\ldots,16\}$ independent Einstein oscillators with spring constant
$k=4.3$ eV/\AA$^2$, mass $m=26.98$ u, and temperature $T=300$ K. A
near-marginal well with stiffness $\alpha=0.066k$ is introduced along one
collective direction, with its offset and gate radii scaled proportionally to
$L_0\sqrt d$, where $L_0=\sqrt{k_BT/k}$ is the thermal length per degree of
freedom.

For each of four seeds, $M=120$ walkers are initialized inside the basin and
propagated for up to $1800$ ps using BAOAB Langevin dynamics
($\Delta t=3$ fs, $\gamma=0.02$ fs$^{-1}$). A walker is considered to have
escaped once $\lVert u-u_t\rVert$ exceeds the outer gate radius, where $u_t$
is the basin center. \Cref{fig:einstein-escape-times} reports the fraction of
walkers whose escape time exceeds $t\in\{25,100,400,1600\}$ ps. At each fixed
observation time, the trapped fraction increases toward one with dimension,
while longer observation times shift the onset of trapping to larger $d$.
This behavior is consistent with the increasing escape times predicted by the
theoretical construction.

\subsection{FCC Argon}
\label{app:fcc-argon-experiment}

We next consider a harmonic face-centred-cubic argon crystal. This setting is useful because the target equilibrium distribution is known exactly, allowing us to compare trajectories generated by a learned potential against a closed-form reference.

We model argon using a Lennard--Jones potential with
$\sigma=3.40$ \AA{} and $\epsilon=0.0104$ eV, smoothly truncated at
$8.5$ \AA{}. Minimizing the Lennard--Jones energy under periodic boundary
conditions gives a relaxed lattice constant $a_0=5.26$ \AA{} and a
nearest-neighbor distance of $3.72$ \AA{}. All simulations are performed at
$T=30$ K ($0.36T_{\mathrm{melt}}$), where the Lindemann ratio is approximately
$5\%$.

The force-constant matrix $K$ is obtained by central finite differences at the
relaxed geometry, symmetrized, and adjusted to satisfy the acoustic sum rule.
After removing the three translational zero modes, the system has $3N-3$
vibrational modes with an upper band edge of approximately
$\nu_{\max}=2.0$ THz. For the harmonic potential
\[
V(u)=\tfrac12u^\top Ku,
\]
the equilibrium distribution on the vibrational subspace is
\(
\ptar=\mathcal N(0,k_BT K^+).
\)
We draw both the training and test configurations independently from this
distribution.

\paragraph{Surrogate model and training:}
The learned potential $\Vhat_\theta(u)$ is represented by an unconstrained
scalar-output $\tanh$ MLP with eight width-$512$ hidden blocks and skip
connections. We deliberately impose no symmetry, locality, or coercivity.
For any fixed network with finite weights, the smooth $\tanh$ activations give
spatially bounded derivatives, although we do not impose a Hessian bound that
is uniform across system sizes. We use Xavier initialization with the $\tanh$ gain, skip connections, and
gradient clipping at norm $1$. Forces are obtained by automatic
differentiation,
\(
\Fhat_\theta=-\nabla\Vhat_\theta.
\)

We minimize the empirical energy-and-force matching objective in
\eqref{eq:energy-force-matching}, normalizing the force term by the force
scale and including a per-degree-of-freedom energy anchor with weight $0.02$.
The network is trained with Adam for $3\times10^4$ steps using a batch size
of $4096$ and a cosine-annealed learning rate initialized at
$2\times10^{-3}$.

\paragraph{Molecular dynamics:}
We integrate \eqref{eq:kinetic-langevin} with BAOAB
\citep{leimkuhler2015}, using $\Delta t=10$ fs and
$\gamma=10^{-3}$ fs$^{-1}$. For each trained model, $256$ walkers are
initialized independently from $\ptar$ and propagated for $25$ ns, with
observables computed over the second half of each trajectory.

We consider system sizes
$N\in\{16,32,\ldots,128\}$ while keeping the training density fixed at
$n/N=625$. Thus, the number of training configurations increases from
$n=10^4$ at $N=16$ to $n=8\times10^4$ at $N=128$. We separately study the
effect of additional data at $N=128$, increasing the training set to
$n=5\times10^5$. Network initialization, training samples, and Langevin noise
are varied independently across four random seeds.

\paragraph{Sampling diagnostics:}
To quantify the discrepancy between the learned dynamics and the target, we
project MD configurations onto the exact phonon modes and compare their
empirical covariance with the target covariance
$\operatorname{diag}(k_BT/\lambda_j)$. We denote the resulting
Gaussian moment-based KL diagnostic by $\widehat{\mathrm{KL}}$. If the exact
moments were known, the corresponding Gaussian divergence would lower-bound
the KL divergence to the Gaussian target. Because we estimate the moments from
finite, temporally correlated trajectories, $\widehat{\mathrm{KL}}$ need not
retain this lower-bound property. We therefore use it only as a
finite-trajectory diagnostic and do not assume that the learned dynamics
converge to a stationary distribution.

We also compute the Debye--Waller factor $\langle u^2\rangle$ using the same
per-atom mean-square-displacement estimator as in the Einstein-crystal
experiment. Although the trajectories themselves are generated using
classical MD, we estimate the heat capacity from the phonon spectrum using
the quantum harmonic expression
\[
C_v=\frac{k_B}{N}\sum_j
\frac{x_j^2e^{x_j}}{(e^{x_j}-1)^2},
\qquad
x_j=\frac{h\nu_j}{k_BT},
\]
where $\{\nu_j\}$ denotes either the true or learned phonon spectrum. We use
this spectrum-based estimate because argon is partially quantum at $30$ K,
so its heat capacity is below the classical value of $3k_B$ per atom.

As a reference, a well-trained surrogate at $N=32$ gives
$\widehat{\mathrm{KL}}=0.004$, while $\langle u^2\rangle$, $C_v$, and
$\nu_{\max}$ agree with the analytic values to within $0.5\%$.

\begin{figure}[!t]
\centering
\begin{subfigure}[t]{0.48\textwidth}
\centering
\includegraphics[width=0.9\linewidth]{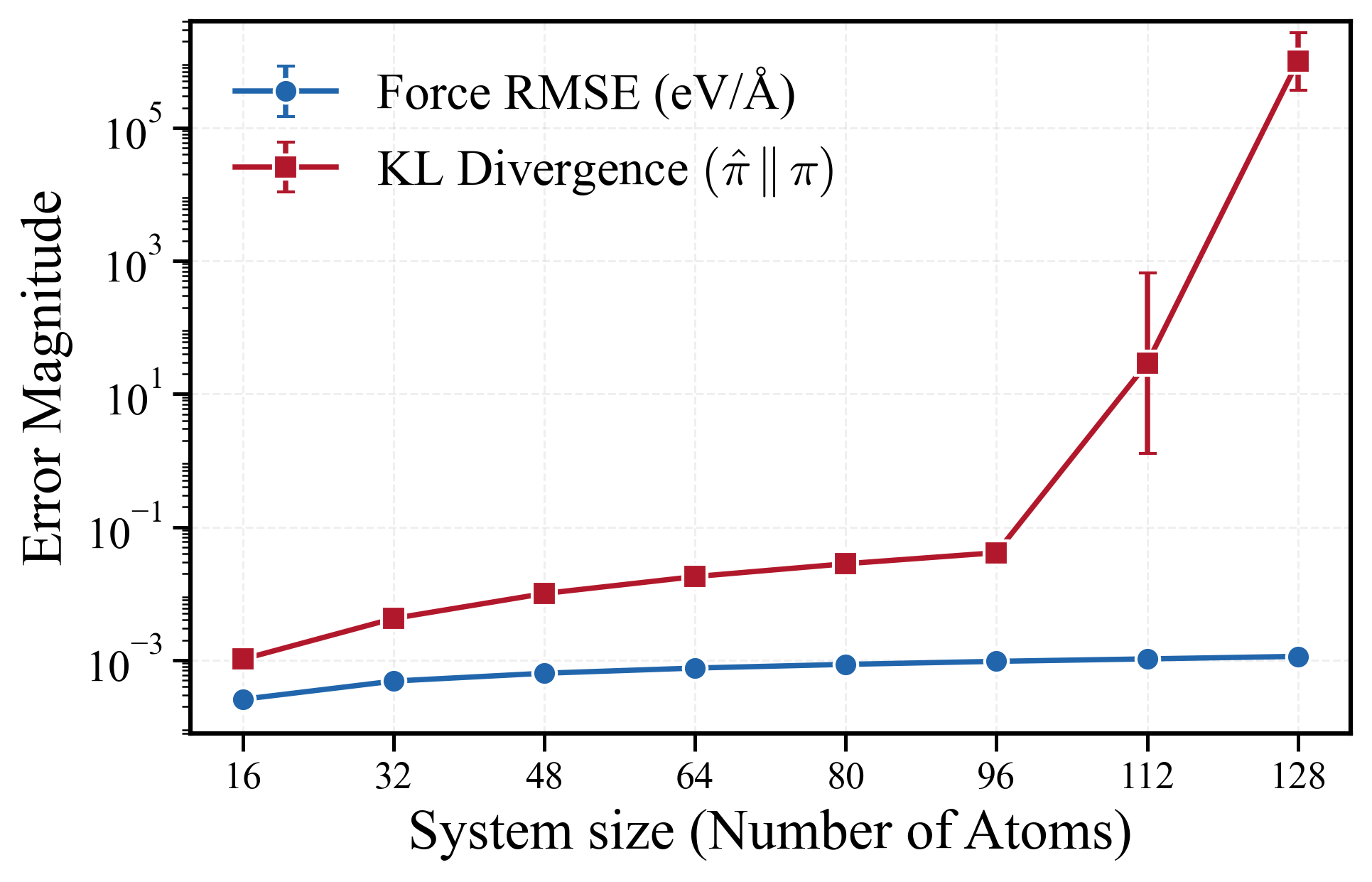}
\caption{Sampling error grows with dimension}
\label{fig:argon-dimension-sweep}
\end{subfigure}
\hfill
\begin{subfigure}[t]{0.48\textwidth}
\centering
\includegraphics[width=0.9\linewidth]{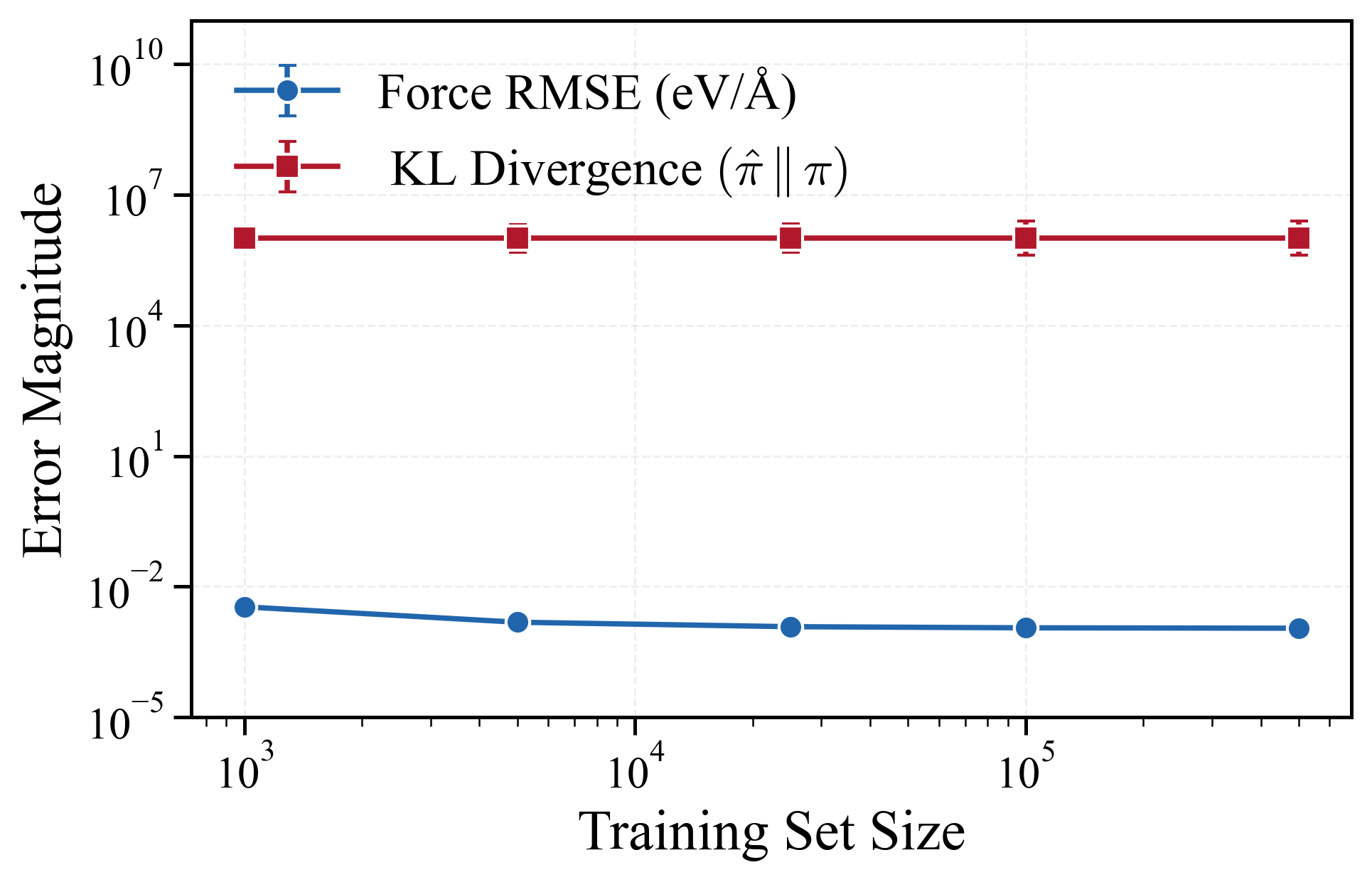}
\caption{Additional data does not improve sampling}
\label{fig:argon-data-sweep}
\end{subfigure}

\caption{Molecular dynamics of the harmonic argon crystal. Error bars show the standard deviation across four random seeds. (a) As the system size increases, the held-out force RMSE remains of order $10^{-3}$ or smaller, while the finite-trajectory Gaussian KL diagnostic increases with dimension. At $N=128$, trajectories from all four seeds leave the crystal basin. (b) At fixed system size $N=128$, increasing the training-set size does not substantially reduce the sampling discrepancy under the tested training protocol.}
\label{fig:argon-sampling-errors}
\vspace{-5mm}
\end{figure}

\paragraph{Sampling error with increasing system size:}
\Cref{fig:argon-dimension-sweep} shows a sharp separation between predictive
and dynamical accuracy. The held-out force error remains of order $10^{-3}$
or smaller throughout the dimension sweep, while the sampling discrepancy
increases by several orders of magnitude. At $N=128$, all four learned
trajectories eventually leave the harmonic crystal basin. This occurs even
though the target distribution is Gaussian and strongly log-concave on the
vibrational subspace.

The experiment therefore illustrates that small held-out force error need not
translate into accurate finite-time sampling as the number of degrees of
freedom increases. Although our asymptotic results apply only beyond a fixed
dimension $d_0$, the same qualitative behavior appears here at system sizes
routinely accessible in molecular simulation.

Increasing the amount of training data does not substantially change this
behavior. As shown in \Cref{fig:argon-data-sweep}, the sampling discrepancy at
$N=128$ remains large as the training set is increased, even though the
surrogate continues to predict forces accurately on held-out equilibrium
configurations. The escaped trajectories enter regions that carry negligible
probability under $\ptar$, where the training objective provides little
control over the learned energy landscape. This behavior is consistent with
the off-distribution freedom used in our theoretical constructions. Related behavior was reported by \citet{fu2023forces}, who observed that, for certain architectures, increasing the training-set size from 1,000 to 90,000 configurations reduced force errors but did not eliminate simulation failures. 

\paragraph{Coercive baselines:}
To examine whether simple global control of the learned potential can mitigate
these failures, we consider two proof-of-concept coercive baselines, C1 and
C2, motivated by \Cref{thm:sampling-guarantees}. In both cases, the learned
potential takes the form
\[
\widehat U(u)=U_0(u)+f_\theta(u),
\]
where $U_0$ is a fixed confining potential. The quadratic or quartic growth of
$U_0$ ensures that $\widehat U$ is confining and has a finite partition
function. These constructions do not, however, guarantee the stronger global
comparison $\widehat U\ge V-\delta$ required by
\Cref{thm:sampling-guarantees}.

For C1, we use an isotropic harmonic prior
\[
U_0(u)=\frac12\kappa\lVert u\rVert^2,
\]
where $\kappa$ is the mean force constant, and train the network on the
residual $V(u)-U_0(u)$.

For C2, we instead use a repulsive boundary potential
\[
U_0(u)=\varepsilon
\left[
\mathrm{ReLU}\left(
\frac{u^\top\Sigma^{-1}u}{k}-c
\right)
\right]^2,
\]
where $\Sigma$ is the empirical covariance of the training configurations.
The prior vanishes inside the ellipsoid
$u^\top\Sigma^{-1}u/k<c$ and therefore leaves the learned potential unchanged
near the training distribution. Outside this region, it grows quartically
with displacement.

\begin{table}[!t]
\centering
\small
\setlength{\tabcolsep}{3.5pt}
\caption{Debye--Waller inflation, finite-trajectory Gaussian KL diagnostic, and phonon-based quantum harmonic heat capacity as functions of system size. C1 reduces the large-$N$ discrepancy, while C2 keeps the displacement statistics close to the target and substantially reduces the KL diagnostic. $C_v$ is computed from the phonon spectrum and averaged over seeds that remain within the harmonic basin. A dash indicates that all seeds left the basin, so a basin-based phonon spectrum is not reported.}
\label{tab:argon-scaling-results}
\begin{tabular}{@{} l ccc ccc cccc @{}}
\toprule
& \multicolumn{3}{c}{$\langle u^2 \rangle / \langle u^2 \rangle_{\mathrm{t}}$}
& \multicolumn{3}{c}{$\widehat{\mathrm{KL}}$}
& \multicolumn{4}{c}{$C_v$ ($k_B$/atom)} \\
\cmidrule(lr){2-4} \cmidrule(lr){5-7} \cmidrule(lr){8-11}
$N$
& Unconst. & C1 & C2
& Unconst. & C1 & C2
& Analytic & Unconst. & C1 & C2 \\
\midrule
16  & 1.00 & 1.00 & 1.00 & 0.001             & 0.001 & 0.001 & 1.901 & 1.901 & 1.901 & 1.901 \\
32  & 1.00 & 1.00 & 1.00 & 0.004             & 0.004 & 0.004 & 1.980 & 1.980 & 1.980 & 1.980 \\
48  & 1.00 & 1.00 & 1.00 & 0.010             & 0.010 & 0.010 & 2.004 & 2.005 & 2.004 & 2.005 \\
64  & 1.00 & 1.00 & 1.00 & 0.018             & 0.018 & 0.018 & 2.019 & 2.020 & 2.019 & 2.020 \\
80  & 1.00 & 1.00 & 1.00 & 0.028             & 0.027 & 0.028 & 2.028 & 2.030 & 2.028 & 2.030 \\
96  & 1.01 & 1.00 & 1.01 & 0.041             & 0.040 & 0.041 & 2.034 & 2.037 & 2.034 & 2.037 \\
112 & 8.90 & 3.30 & 1.01 & 29.1              & 7.20  & 0.057 & 2.039 & 2.041 & 2.033 & 2.041 \\
128 & 641  & 2.25 & 1.02 & $1.5{\times}10^6$ & 5.70  & 0.590 & 2.042 & -- & 2.085 & 2.048 \\
\bottomrule
\end{tabular}
\end{table}

The results in \Cref{tab:argon-scaling-results} suggest that even simple
confining tails can substantially change the behavior of the learned dynamics.
C1 reduces the large-$N$ discrepancy but still shows appreciable displacement
inflation. C2 keeps the mean-square displacement close to the target across
the dimension sweep and produces substantially smaller KL diagnostics at the
largest system sizes. These results are consistent with the role of global
energy control in \Cref{thm:sampling-guarantees}, although neither baseline
certifies the assumptions of the theorem.

\paragraph{Implementation:}
The experiments were implemented in PyTorch, NumPy, and ASE using the
Lennard--Jones calculator. Each run used a single NVIDIA A100 GPU. The
computational cost was dominated by the $N=128$ simulations, for which
propagating $256$ walkers for $2.5\times10^6$ integration steps required
approximately one hour per run.

\subsection{Limitations}

These experiments are restricted to controlled harmonic systems
and unconstrained neural networks. They therefore do not establish how often
analogous failures occur in practical MLIP training or deployment. The
sampling discrepancies reported above are Gaussian moment-based diagnostics
estimated from finite and temporally correlated trajectories. They need not 
measure the full distributional divergence. In particular, the unconstrained neural network need
not define a normalizable Boltzmann distribution, so trajectory escape should
not be interpreted as a direct numerical verification of
\Cref{thm:equilibrium-mismatch}.

The training-set-size experiment likewise describes the behavior of the
specific architecture and optimization procedure studied here rather than a
general impossibility result about additional data. Similarly, the coercive
baselines demonstrate the effect of introducing confining tails but do not
guarantee the global energy comparison required by
\Cref{thm:sampling-guarantees}. The experiments should therefore be viewed as
controlled illustrations of the mechanisms suggested by the theory. The
formal impossibility and sampling guarantees follow from the theoretical
results themselves.
\begin{figure}[!htbp]
    \centering
    \includegraphics[width=0.7\linewidth]{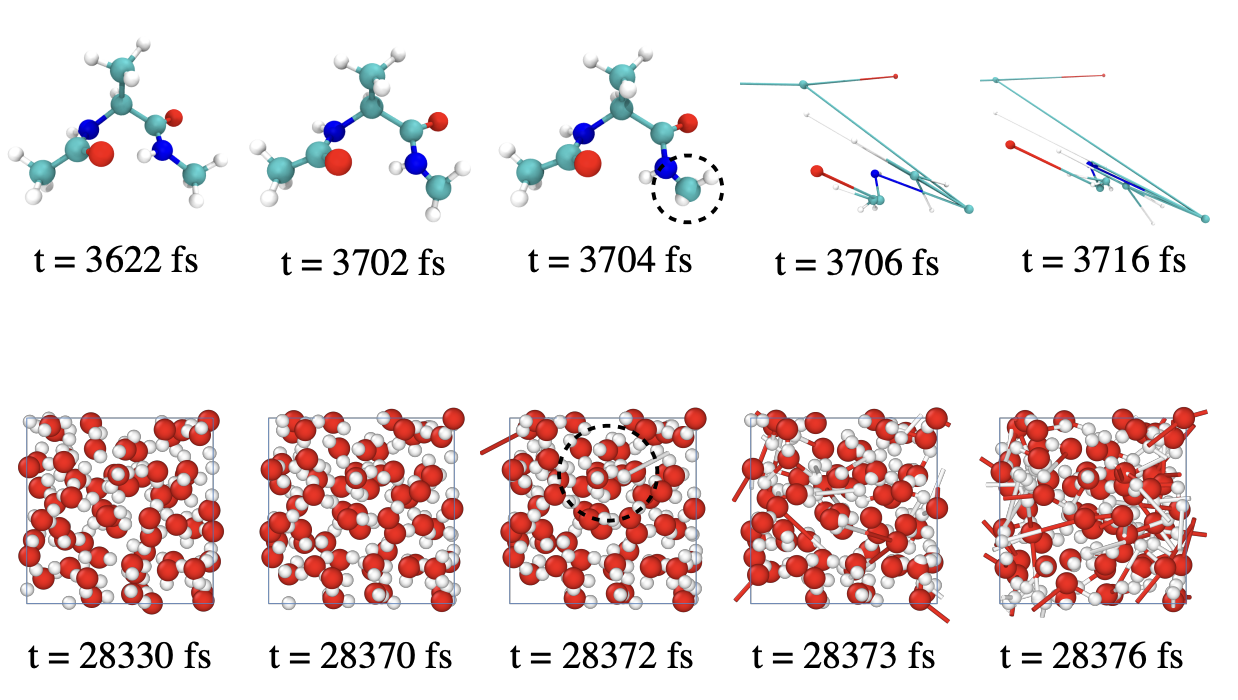}
    \caption{Figure adapted from \citet{fu2023forces} (Figure 6) depicting simulation collapse when applying (a) NequIP to alanine
dipeptide and (b) GemNet-T to water.}
    \label{fig:forces}
\end{figure}
\section{Technical preliminaries}
We first introduce several key technical preliminaries, beginning with two concentration-of-measure results in high dimensions.
\label{app:technical-preliminaries}

\subsection{Concentration of measure}
\label{app:concentration}

\begin{lemma}[Gaussian norm concentration; \citealp{vershynin2018}]
\label{lem:gaussian-norm-concentration}
There are universal constants $c,C>0$ such that for $X\sim\N(0,\Id)$ and all $t>0$,
\[
\Prob\big(\,\abs{\norm X-\sqrt d}\ge t\sqrt d\,\big)\ \le\ Ce^{-ct^2d}.
\]
\end{lemma}

\begin{lemma}[Lipschitz concentration; \citealp{bakry2014analysis}]
\label{lem:lipschitz-concentration}
Let $f:\R^d\to\R$ be $L$-Lipschitz. If $X\sim\N(0,\Id)$ then
$\Prob\big(\abs{f(X)-\E f(X)}\ge t\big)\le2e^{-t^2/(2L^2)}$ for all $t>0$. More generally, if the
law of $X$ is $\alpha$-strongly log-concave the same bound holds with $e^{-\alpha t^2/(2L^2)}$ in
place of $e^{-t^2/(2L^2)}$; this yields the strongly-log-concave extension of \Cref{rem:inverse-temperature-scaling}.
\end{lemma}
We note two consequences that will be used later in our analysis.

First, if $x_0 \sim \mathcal{N}(0, I_d)$, then \Cref{lem:gaussian-norm-concentration} with $t = 0.1$ gives
\[
    \mathbb{P} \big( \lVert x_0 \rVert \le 1.1\sqrt{d} \big) \ge 1 - C e^{-cd},
\]
and likewise for $v_0 \sim \mathcal{N}(0, I_d)$. 

Second, let $\ptar = \mathcal{N}(\mu, \mathrm{Id})$ with $\lVert \mu \rVert = \xi\sqrt{d}$ for $\xi > 5$, and fix $t \in (0, \xi - 1 - b_\xi)$. By \Cref{lem:gaussian-norm-concentration}, we have
\[
    \lVert x - \mu \rVert \le (1+t)\sqrt{d} \qquad \text{with probability } 1 - e^{-\Omega(d)},
\]
and on this event,
\[
    \lVert x \rVert \ge \lVert \mu \rVert - \lVert x - \mu \rVert \ge (\xi - 1 - t)\sqrt{d} > b_\xi \sqrt{d}.
\]
Hence, we obtain the bounds
\[
    \ptar \big( \lVert x \rVert \ge b_\xi \sqrt{d} \big) \ge 1 - e^{-\Omega(d)} \qquad \text{and, in particular} \qquad \ptar \big( \lVert x \rVert \ge 4\sqrt{d} \big) \ge 1 - e^{-\Omega(d)}.
\]

\subsection{Bump functions}
\label{app:bump-functions}
\begin{lemma}
\label{lem:smooth-bump-function}
There is a universal constant $C_0$ with the following property. For all $0 < a < b$ and all $d \ge 1$, there is a radial $\psi_{a,b} \in C^\infty(\mathbb{R}^d; [0,1])$ with $\psi_{a,b} \equiv 1$ on $\{\lVert y \rVert \le a\}$, $\psi_{a,b} \equiv 0$ on $\{\lVert y \rVert \ge b\}$, and
\begin{equation}
    \lVert \nabla \psi_{a,b} \rVert_\infty \le \frac{C_0}{b-a}, \qquad \lVert \nabla^2 \psi_{a,b} \rVert_{\mathrm{op},\infty} \le \frac{C_0}{(b-a)^2} + \frac{C_0}{a(b-a)}.
\end{equation}
In particular, these bounds are independent of $d$.
\end{lemma}

\begin{proof}
Let 
\[
    f(u) = e^{-1/u} \mathbf{1}_{\{u>0\}} \qquad \text{and} \qquad g_0(u) = \frac{f(1-u)}{f(1-u) + f(u)}.
\]
The denominator is strictly positive for every $u \in \mathbb{R}$, so $g_0 \in C^\infty(\mathbb{R}; [0,1])$ with $g_0 \equiv 1$ on $(-\infty, 0]$, $g_0 \equiv 0$ on $[1, \infty)$, and universal bounds
\[
    \lVert g_0' \rVert_\infty, \, \lVert g_0'' \rVert_\infty \le C_0.
\]
Put
\[
    g(u) := g_0 \left( \frac{u-a}{b-a} \right) \qquad \text{and} \qquad \psi_{a,b}(y) := g(\lVert y \rVert).
\]
Since $\psi_{a,b} \equiv 1$ on a neighborhood of the origin and $y \mapsto \lVert y \rVert$ is smooth on $\mathbb{R}^d \setminus \{0\}$, we have $\psi_{a,b} \in C^\infty(\mathbb{R}^d)$. Writing $r = \lVert y \rVert$ and $\hat{y} = y/r$,
\begin{equation}
    \nabla \psi_{a,b}(y) = g'(r) \, \hat{y}, \qquad \nabla^2 \psi_{a,b}(y) = g''(r) \, \hat{y} \hat{y}^\top + \frac{g'(r)}{r} \big( \mathrm{Id} - \hat{y} \hat{y}^\top \big),
\end{equation}
and $g', g''$ vanish outside $\{a \le r \le b\}$, where $r \ge a$. The bounds follow directly from
\[
    \lVert g' \rVert_\infty \le \frac{C_0}{b-a} \qquad \text{and} \qquad \lVert g'' \rVert_\infty \le \frac{C_0}{(b-a)^2}.
\]
The eigenvalues of $\nabla^2 \psi_{a,b}$ are $g''(r)$ and $g'(r)/r$ irrespective of dimension, so no estimate depends on $d$.
\end{proof}

\Cref{thm:kinetic-trapping} uses $\psi_{4\sqrt d,\,b_\xi\sqrt d}$ with $b_\xi:=\tfrac12(3+\xi)$, which
satisfies $4<b_\xi<\xi-1$ for every $\xi>5$. Since $b_\xi-4=\tfrac12(\xi-5)$, the gradient and
Hessian bounds are $C_\xi/\sqrt d$ and $C_\xi/d$, with $C_\xi$ depending only on $\xi$ (of order
$(\xi-5)^{-1}$ and $(\xi-5)^{-2}$ respectively). \Cref{thm:data-interpolation} uses
$\psi_{0.15\sqrt d,\,0.16\sqrt d}(\,\cdot-x_i)$, with bounds $O(1/\sqrt d)$ and $O(1/d)$, and
\Cref{thm:equilibrium-mismatch} uses $\psi_{\theta\sqrt d,\,2\theta\sqrt d}(\,\cdot-x_\star)$, with bounds
$C_0/(\theta\sqrt d)$ and $2C_0/(\theta^2 d)$.

\section{The kinetic exit-time lemma}
\label{app:kinetic-exit}
The following lemma quantifies how long a stiff underdamped oscillator remains inside a high-dimensional ball. It is the dynamical input used to turn the localized potential constructions in \Cref{app:kinetic-trapping-proof,app:data-initialization-proof} into exponentially long trapping statements.

\begin{lemma}\label{lem:kinetic-exit-time}
Fix $\gamma>0$ and constants $r>\rho_x\ge0$, $\rho_v\ge0$, and write $m:=r-\rho_x$. Let
$(x_t,v_t)_{t\ge0}$ solve the linear SDE
\[
\dd x_t=v_t\,\dd t,\qquad \dd v_t=(-\alpha x_t-\gamma v_t)\,\dd t+\sqrt{2\gamma}\,\dd B_t,
\]
where $B_t$ is a standard $d$-dimensional Brownian motion and the deterministic initial condition
satisfies $\norm{x_0}\le\rho_x\sqrt d$ and $\norm{v_0}\le\rho_v\sqrt d$. Let
$\tau=\inf\{t\ge0:\norm{x_t}\ge r\sqrt d\}$. There is a universal constant $C'$ such that, setting
\[
\alpha_\star\ :=\ \max\Big\{\frac{\gamma^2}{2},\ \frac{C'\big(1+\gamma^2\rho_x^2+\rho_v^2\big)}{m^2}\Big\},
\qquad
c_\star\ :=\ \frac{\alpha_\star m^2}{64},
\]
the following holds: there are $C,K_\star>0$, depending only on $\gamma,r,\rho_x,\rho_v$, such that
for all $\alpha\ge\alpha_\star$, $d\ge K_\star$, and $T\le e^{c_\star d}$,
\[
\Prob(\tau\le T)\ \le\ Ce^{-c_\star d}.
\]
The bound is uniform over all deterministic initial data satisfying the stated norm constraints.
\end{lemma}

\begin{proof}
Throughout, we write
\begin{equation}
\label{eq:exit-margin}
    m := r - \rho_x > 0
\end{equation}
for the available margin. Consider one Cartesian coordinate; the $d$ coordinates are independent and identically distributed. The scalar process solves $\ddot{x} = -\alpha x - \gamma \dot{x} + \sqrt{2\gamma}\,\dot{B}$. 

In the underdamped regime $\alpha > \gamma^2/4$, we set $\omega = \sqrt{\alpha - \gamma^2/4}$. The Green's function of $\ddot{G} + \gamma \dot{G} + \alpha G = \delta$ is given and bounded by
\begin{align}
\label{eq:langevin-green-function}
    G(u) &= \frac{1}{\omega} e^{-\gamma u/2} \sin(\omega u) \quad (u \ge 0), \nonumber \\
    \lvert G(u) \rvert &\le \frac{1}{\omega} e^{-\gamma u/2}
\end{align}
Solving by variation of parameters, and writing $h_0, h_1$ for the position responses to a unit initial position and velocity, we have
\begin{align}
    x_t &= h_0(t)x_0 + h_1(t)v_0 + \sqrt{2\gamma} \underbrace{ \int_0^t G(t-s) \, \mathrm{d}B_s }_{=: W_t}, \label{eq:langevin-position-solution} \\
    h_0(t) &= e^{-\gamma t/2} \left( \cos \omega t + \frac{\gamma}{2\omega} \sin \omega t \right), \nonumber \\
    h_1(t) &= G(t), \nonumber
\end{align}
where $W_t \in \mathbb{R}^d$ has i.i.d.\ coordinates.

\paragraph{Stationary Scale:}
Using $\gamma^2 + 4\omega^2 = 4\alpha$ and the identity $\sin^2(\theta) = \frac{1}{2}(1 - \cos 2\theta)$, we can explicitly compute the integral:
\begin{align}
\label{eq:green-function-variance}
    \int_0^\infty G(r)^2 \, \mathrm{d}r 
    &= \frac{1}{2\omega^2} \left( \frac{1}{\gamma} - \frac{\gamma}{\gamma^2 + 4\omega^2} \right) \nonumber \\
    &= \frac{1}{2\omega^2} \cdot \frac{4\omega^2}{4\alpha\gamma} \nonumber \\
    &= \frac{1}{2\alpha\gamma}
\end{align}
Therefore, each coordinate of $\sqrt{2\gamma} W_t$ has variance $2\gamma \int_0^t G^2 \le 1/\alpha$, uniformly in $t$. Equivalently, applying the equipartition theorem $\frac{1}{2}\alpha \langle x^2 \rangle = \frac{1}{2}$ yields $\langle x^2 \rangle = 1/\alpha$.

\paragraph{Homogeneous Part:}
Assume $\alpha \ge \gamma^2/2$, so that $\omega \ge \sqrt{\alpha/2}$, which implies $\frac{\gamma}{2\omega} \le \frac{\gamma}{\sqrt{2\alpha}}$ and $\frac{1}{\omega} \le \sqrt{2/\alpha}$. 

With the initial bounds $\lVert x_0 \rVert \le \rho_x \sqrt{d}$, $\lVert v_0 \rVert \le \rho_v \sqrt{d}$, and $e^{-\gamma t/2} \le 1$, we obtain the following bound on the homogeneous evolution:
\begin{align}
\label{eq:homogeneous-component-bound}
    \lVert h_0(t)x_0 + h_1(t)v_0 \rVert 
    &\le \left( 1 + \frac{\gamma}{\sqrt{2\alpha}} \right) \rho_x \sqrt{d} + \sqrt{\frac{2}{\alpha}} \, \rho_v \sqrt{d} \nonumber \\
    &= (\rho_x + \delta_\alpha) \sqrt{d},
\end{align}
where we have defined
\begin{equation*}
    \delta_\alpha := \frac{\gamma \rho_x}{\sqrt{2\alpha}} + \frac{\sqrt{2} \, \rho_v}{\sqrt{\alpha}}.
\end{equation*}
Since $\delta_\alpha \to 0$ as $\alpha \to \infty$, we may fix a stiffness threshold $\alpha_1$:
\begin{align}
\label{eq:initial-stiffness-threshold}
    \alpha_1 &:= \frac{16}{m^2} \left( \frac{\gamma \rho_x}{\sqrt{2}} + \sqrt{2} \, \rho_v \right)^2, \\
    \text{so that} \qquad \delta_\alpha &\le \frac{m}{4} \qquad \text{for all } \alpha \ge \alpha_1. \nonumber
\end{align}
The homogeneous part is then at most $(\rho_x + \frac{m}{4})\sqrt{d}$ for every $t \ge 0$. Note that the velocity enters only through $\rho_v / \sqrt{\alpha}$. A large initial velocity is admissible provided $\alpha$ is taken correspondingly large, which is precisely the mechanism behind the choice of \eqref{eq:initial-stiffness-threshold}.

\paragraph{Stochastic Part:}
Write $X_t := \sqrt{2\gamma} W_t$ and $Z_{t,u} := \langle X_t, u \rangle$ for $u \in \mathbb{S}^{d-1}$. This defines a centered Gaussian process on $\mathcal{T} := [0,T] \times \mathbb{S}^{d-1}$ with variance bounded by
\begin{equation}
\label{eq:stochastic-variance-bound}
    \sigma^2 := \sup_{t,u} \mathbb{E} Z_{t,u}^2 \le \frac{1}{\alpha}
\end{equation}
by \eqref{eq:green-function-variance}, and with $\sup_{t \le T} \lVert X_t \rVert = \sup_{\mathcal{T}} Z_{t,u}$. We bound the supremum over a unit window $\mathcal{T}_1 := [0,1] \times \mathbb{S}^{d-1}$ and then tile the space.

For $0 \le t_1 \le t_2$ with $h := t_2 - t_1 \le 1$ and $u_1, u_2 \in \mathbb{S}^{d-1}$, It\^o's isometry gives
\begin{align}
    \mathbb{E} \big( Z_{t_1, u_1} - Z_{t_2, u_2} \big)^2 
    &= 2\gamma \int_0^{t_1} \lVert G(t_1 - s)u_1 - G(t_2 - s)u_2 \rVert^2 \, \mathrm{d}s + 2\gamma \int_{t_1}^{t_2} G(t_2 - s)^2 \, \mathrm{d}s \nonumber \\
    &\le \underbrace{ 4\gamma \int_0^\infty G^2 }_{= 2/\alpha} \lVert u_1 - u_2 \rVert^2 + 4\gamma \int_0^\infty \big( G(\rho) - G(\rho + h) \big)^2 \, \mathrm{d}\rho + 2\gamma \int_0^h G^2. \label{eq:gaussian-process-metric}
\end{align}
Since $G'(u) = e^{-\gamma u/2} \left( \cos \omega u - \frac{\gamma}{2\omega} \sin \omega u \right)$ satisfies $\lvert G'(u) \rvert \le 2 e^{-\gamma u/2}$ for $\alpha \ge \gamma^2/2$, we have $\int_0^\infty (G')^2 \le 4/\gamma$ and hence, by Cauchy--Schwarz and Fubini, $\int_0^\infty (G(\rho) - G(\rho + h))^2 \, \mathrm{d}\rho \le h^2 \int_0^\infty (G')^2 \le 4h^2/\gamma$. Moreover, $\int_0^h G^2 \le h/\omega^2 \le 2h/\alpha$. Substituting these into \eqref{eq:gaussian-process-metric} yields the metric bounds:
\begin{align}
\label{eq:gaussian-process-metric-bound}
    d \big( (t_1, u_1), (t_2, u_2) \big)^2 &\le \frac{2}{\alpha} \lVert u_1 - u_2 \rVert^2 + 16h^2 + \frac{4\gamma h}{\alpha}, \\
    \text{and globally,} \qquad d &\le \frac{2}{\sqrt{\alpha}} \nonumber
\end{align}
by \eqref{eq:stochastic-variance-bound}. Consequently, $d \le \varepsilon$ holds whenever $\lVert u_1 - u_2 \rVert \le \varepsilon\sqrt{\alpha}/\sqrt{6}$ and $h \le c \min\{\varepsilon, \varepsilon^2\alpha/\gamma\}$. 

So with $\mathrm{diam}(\mathcal{T}_1) \le 2/\sqrt{\alpha}$, the covering numbers obey
\begin{align}
\label{eq:covering-number-bound}
    N \big( \mathcal{T}_1, d, \varepsilon \big) 
    &\le N \big( \mathbb{S}^{d-1}, \lVert \cdot \rVert, c\varepsilon\sqrt{\alpha} \big)
    \cdot N \big( [0,1], \lvert \cdot \rvert, c \min \{ \varepsilon, \varepsilon^2\alpha/\gamma \} \big) \nonumber \\
    &\le \left( \frac{C}{\varepsilon\sqrt{\alpha}} \right)^d
    \left(1+\frac{C}{\varepsilon}+\frac{C\gamma}{\varepsilon^2\alpha}\right).
\end{align}
The second factor is the contribution of the time direction. Unlike the sphere, it carries no exponential factor in $d$.

Dudley's chaining with the substitution $\varepsilon = w/\sqrt{\alpha}$ gives
\begin{align}
\label{eq:expected-supremum-bound}
    \mathbb{E} \sup_{\mathcal{T}_1} Z_{t,u} 
    &\le C \int_0^{2/\sqrt{\alpha}} \sqrt{ \log N(\mathcal{T}_1, d, \varepsilon) } \, \mathrm{d}\varepsilon \nonumber \\
    &\le \frac{C}{\sqrt{\alpha}} \int_0^2
    \left(
    \sqrt{ d \log \frac{C}{w} }
    +
    \sqrt{
    \log\left(
    1+\frac{C\sqrt{\alpha}}{w}+\frac{C\gamma}{w^2}
    \right)}
    \right)\mathrm{d}w \nonumber \\
    &\le \frac{C_D\sqrt d}{\sqrt{\alpha}} + R(\alpha,\gamma),
\end{align}
where
\[
R(\alpha,\gamma)
:=
\frac{C}{\sqrt{\alpha}}
\int_0^2
\sqrt{
\log\left(
1+\frac{C\sqrt{\alpha}}{w}+\frac{C\gamma}{w^2}
\right)}
\,\mathrm{d}w.
\]
An elementary bound gives
\[
R(\alpha,\gamma)
\le
\frac{C}{\sqrt{\alpha}}
\left(
1+\sqrt{\log(1+\alpha)}+\sqrt{\log(1+\gamma)}
\right),
\]
so $R(\alpha,\gamma)\to0$ as $\alpha\to\infty$.

We fix a second stiffness threshold:
\begin{align}
\label{eq:stochastic-stiffness-threshold}
    \alpha_2 &:= \frac{64 C_D^2}{m^2}, \\
    \text{so that} \qquad \frac{C_D}{\sqrt{\alpha}} &\le \frac{m}{8} \qquad \text{for all } \alpha \ge \alpha_2. \nonumber
\end{align}
Let
\[
R_\star
:=
\sup_{\alpha\ge\max\{\gamma^2/2,\alpha_2\}}
R(\alpha,\gamma)
<\infty,
\qquad
K_\star
:=
\max\left\{1,\frac{64R_\star^2}{m^2}\right\}.
\]
Then for every $\alpha\ge\max\{\gamma^2/2,\alpha_2\}$ and $d\ge K_\star$,
\[
\mathbb{E}\sup_{\mathcal{T}_1}Z_{t,u}
\le
\frac{m}{4}\sqrt d.
\]
By the translation invariance of \eqref{eq:gaussian-process-metric-bound} in $t$ and the uniform variance bound \eqref{eq:stochastic-variance-bound}, the same bound holds with $[0,1]$ replaced by any unit window $[k,k+1]$.

The Borell--TIS inequality \citep{adler2007}, applied on each window with $\sigma^2 \le 1/\alpha$ and deviation $s = \frac{m}{4}\sqrt{d}$, gives
\begin{equation}
    \mathbb{P} \left( \sup_{[k, k+1] \times \mathbb{S}^{d-1}} Z_{t,u} \ge \frac{m}{2}\sqrt{d} \right) \le \exp \left( -\frac{\alpha m^2 d}{32} \right).
\end{equation}

\paragraph{Conclusion:}
Set the overall required stiffness and corresponding exponential rate as:
\begin{align}
\label{eq:combined-stiffness-threshold}
    \alpha_\star &:= \max \left\{ \frac{\gamma^2}{2}, \, \alpha_1, \, \alpha_2 \right\} 
    = \max \left\{ \frac{\gamma^2}{2}, \, \frac{C(1 + \gamma^2 \rho_x^2 + \rho_v^2)}{(r - \rho_x)^2} \right\}, \\
    c_\star &:= \frac{\alpha_\star m^2}{64}. \nonumber
\end{align}
Let $\alpha \ge \alpha_\star$, $d \ge K_\star$, and $T \le e^{c_\star d}$. A union bound over the $\lceil T \rceil \le e^{c_\star d} + 1$ unit windows bounds the stochastic term:
\begin{align}
    \mathbb{P} \left( \sup_{t \le T} \lVert X_t \rVert \ge \frac{m}{2}\sqrt{d} \right) 
    &\le (e^{c_\star d} + 1) e^{-\alpha m^2 d/32} \nonumber \\
    &\le 2e^{c_\star d} e^{-2c_\star d} \nonumber \\
    &= 2e^{-c_\star d},
\end{align}
using the fact that $\alpha m^2 / 32 \ge \alpha_\star m^2 / 32 = 2c_\star$. 

On the complementary event, combining this stochastic bound with the homogeneous bound from \eqref{eq:homogeneous-component-bound} and \eqref{eq:initial-stiffness-threshold}, we find that the total position is strictly bounded:
\begin{align}
    \sup_{t \le T} \lVert x_t \rVert 
    &\le \left( \rho_x + \frac{m}{4} \right) \sqrt{d} + \frac{m}{2}\sqrt{d} \nonumber \\
    &= \left( \rho_x + \frac{3m}{4} \right) \sqrt{d} \nonumber \\
    &< (\rho_x + m)\sqrt{d} \nonumber \\
    &= r\sqrt{d}.
\end{align}
Because the trajectory strictly remains inside radius $r\sqrt{d}$, we have $\tau > T$. Hence, $\mathbb{P}(\tau \le T) \le C e^{-c_\star d}$.
\end{proof}

\begin{remark}
\label{rem:exit-radius-scaling}
Whenever the second term in \eqref{eq:combined-stiffness-threshold} attains the maximum, $\alpha_\star$ scales as $(r - \rho_x)^{-2}$, while $c_\star = \alpha_\star(r - \rho_x)^2/64$ does not depend on the radii: at the threshold stiffness, $c_\star$ is a constant determined by $\gamma\rho_x$ and $\rho_v$. The exponential time horizon $e^{c_\star d}$ therefore does not deteriorate as the exit radius shrinks, provided $\alpha$ is increased accordingly. This is what allows the lemma to be applied at $r=4$ in \Cref{thm:kinetic-trapping} and at $r=0.10$ in \Cref{thm:data-interpolation} without losing the exponential-in-$d$ time horizon.
\end{remark}

\section{Proof of the kinetic-trapping theorem}
\label{app:kinetic-trapping-proof}

\begin{proof}[Proof of \Cref{thm:kinetic-trapping}]

Recall $V(x) = \frac{1}{2} \lVert x - \mu \rVert^2$ with $\lVert \mu \rVert = \xi\sqrt{d}$ and $\xi > 5$, so that $\ptar = \mathcal{N}(\mu, I_d)$. We set
\begin{equation}
\label{eq:shift-radius}
    b_\xi := \frac{1}{2}(3 + \xi), \qquad \text{so that} \qquad 4 < b_\xi < \xi - 1,
\end{equation}
both inequalities being equivalent to $\xi > 5$. We fix
\begin{equation}
\label{eq:trap-stiffness}
    \alpha \ \ge\ \max\big\{ \alpha_\star(\gamma, 4, 1.1, 1.1), \ \tfrac{1}{16}(b_\xi + \xi)^2 \big\},
\end{equation}
with $\alpha_\star$ as in \Cref{lem:kinetic-exit-time}; both terms are constants depending only on $\gamma$ and $\xi$. The second is used only in \Cref{app:sufficient-conditions-proof}.

\paragraph{Construction:}
Let $\chi := \psi_{4\sqrt{d}, \, b_\xi\sqrt{d}}$ be the radial cutoff of \Cref{lem:smooth-bump-function}, so that $\chi \equiv 1$ on $\{\lVert x \rVert \le 4\sqrt{d}\}$ and $\chi \equiv 0$ on $\{\lVert x \rVert \ge b_\xi\sqrt{d}\}$. We put
\begin{equation}
\label{eq:kinetic-trap-potential}
    U(x) := \frac{\alpha}{2} \lVert x \rVert^2, \qquad \Vhat(x) := V(x) + \chi(x) \big( U(x) - V(x) \big), \qquad \Fhat := -\nabla \Vhat.
\end{equation}
Then $\Vhat \in C^\infty(\mathbb{R}^d)$ and $\Fhat$ is conservative by construction. Moreover, $\Fhat(x) = -\alpha x$ on the region $\{\lVert x \rVert \le 4\sqrt{d}\}$, while $\Vhat = V$ and $\Fhat = F$ on $\{\lVert x \rVert \ge b_\xi\sqrt{d}\}$.

\paragraph{Regularity and well-posedness:}
On the set $B := \{\lVert x \rVert \le b_\xi\sqrt{d}\}$, we have $\lVert x - \mu \rVert \le (b_\xi + \xi)\sqrt{d}$. Hence, the pointwise gaps are bounded by
\begin{equation}
\label{eq:kinetic-trap-potential-regions}
    \lvert U(x) - V(x) \rvert = O_\xi(\alpha d), \qquad \lVert \nabla U(x) - \nabla V(x) \rVert = O_\xi(\alpha\sqrt{d}) \qquad \text{on } B,
\end{equation}
while by \Cref{lem:smooth-bump-function} we have $\lVert \nabla \chi \rVert = O_\xi(1/\sqrt{d})$ and $\lVert \nabla^2 \chi \rVert_{\mathrm{op}} = O_\xi(1/d)$. Differentiating \eqref{eq:kinetic-trap-potential} twice yields
\begin{align*}
    \nabla^2 \Vhat 
    &= \nabla^2 V + \chi \big( \nabla^2 U - \nabla^2 V \big) \\
    &\quad + \nabla \chi \, (\nabla U - \nabla V)^\top + (\nabla U - \nabla V) \nabla \chi^\top \\
    &\quad + (U - V) \nabla^2 \chi.
\end{align*}
Each term in this expansion is $O_\xi(\alpha)$ by \eqref{eq:kinetic-trap-potential-regions}. Hence, $\sup_x \lVert \nabla^2 \Vhat(x) \rVert_{\mathrm{op}} \le H$ for a constant $H = O_\xi(\alpha)$ that is independent of $d$. Consequently, $\Fhat$ is globally Lipschitz, and \eqref{eq:kinetic-langevin} has a unique strong solution by \citet[Thm.~5.2.1]{oksendal2013}.

\paragraph{Target mass of the modified region:}
Write $x = \mu + Z$ with $Z \sim \mathcal{N}(0, I_d)$. If $\lVert x \rVert \le b_\xi\sqrt{d}$, then
\[
    \lVert Z \rVert \ge \lVert \mu \rVert - \lVert x \rVert \ge (\xi - b_\xi)\sqrt{d},
\]
and we know $\xi - b_\xi = \frac{1}{2}(\xi - 3) > 1$ by \eqref{eq:shift-radius}. Applying \Cref{lem:gaussian-norm-concentration} with $t = \frac{1}{2}(\xi - 5) > 0$, we bound the mass of the modified region:
\begin{align}
\label{eq:modified-region-target-mass}
   \ptar(B) 
    &\le \mathbb{P} \left( \lVert Z \rVert \ge \left( 1 + \frac{\xi - 5}{2} \right) \sqrt{d} \right) \nonumber \\
    &\le C e^{-c(\xi - 5)^2 d / 4} \nonumber \\
    &= e^{-\Omega(d)}
\end{align}
The exact same argument, but with $4$ in place of $b_\xi$ and using the fact that $\xi - 4 > 1$, gives
\begin{equation}
\label{eq:trap-region-target-mass}
    \ptar(A) \ge 1 - e^{-\Omega(d)}, \qquad \text{where} \qquad A := \{ \lVert x \rVert \ge 4\sqrt{d} \}.
\end{equation}

\paragraph{Population $L^p$ errors:}
Both $\Vhat - V$ and $\Fhat - F$ vanish outside of $B$. On $B$, they are bounded by \eqref{eq:kinetic-trap-potential-regions} together with the $\nabla \chi$ contribution $\lvert U - V \rvert \lVert \nabla \chi \rVert = O_\xi(\alpha\sqrt{d})$. Hence, for any fixed $p \ge 1$, we can bound the population error for the forces:
\begin{align}
    \mathbb{E}_{x \sim \ptar} \big[ \lVert \Fhat(x) - F(x) \rVert^p \big]^{1/p} 
    &\le O_\xi(\alpha\sqrt{d}) \, \ptar(B)^{1/p} \nonumber \\
    &\le O_\xi(\alpha\sqrt{d}) \, e^{-\Omega(d)/p} \nonumber \\
    &= e^{-\Omega(d)}
\end{align}
Identically, we find $\mathbb{E}_{x \sim \ptar} \big[ \lvert \Vhat(x) - V(x) \rvert^p \big]^{1/p} \le O_\xi(\alpha d) \, e^{-\Omega(d)/p} = e^{-\Omega(d)}$. In both cases, the bounds hold for all $d \ge d_0(p, \xi)$, as the polynomial prefactors are naturally absorbed once $d$ is large relative to $p$. This concludes part (i).

\paragraph{Total-variation lower bound:}
Define the high-probability initialization event
\[
    G_0 := \big\{ \lVert x_0 \rVert \le 1.1\sqrt{d}, \ \lVert v_0 \rVert \le 1.1\sqrt{d} \big\}.
\]
By \Cref{lem:gaussian-norm-concentration}, $\mathbb{P}(G_0) \ge 1 - e^{-\Omega(d)}$. While $\lVert x_t \rVert \le 4\sqrt{d}$, we have $\Fhat(x_t) = -\alpha x_t$. Therefore, up to the exit time $\tau = \inf \{ t : \lVert x_t \rVert \ge 4\sqrt{d} \}$, the process coincides in law with the linear kinetic Langevin dynamics analyzed in \Cref{lem:kinetic-exit-time}. 

That bound holds uniformly over any deterministic initial data within the ball of \eqref{eq:homogeneous-component-bound}. Hence, it also holds after conditioning on $G_0$ and integrating over $(x_0, v_0)$. Applying \Cref{lem:kinetic-exit-time} with $r = 4$ and $\rho_x = \rho_v = 1.1$ gives
\[
    \mathbb{P}(\tau \le T \mid G_0) \le C e^{-c_\star d} \qquad \text{for every } T \le e^{c_\star d},
\]
and therefore $\mathbb{P}(x_T \in A \mid G_0) \le C e^{-c_\star d}$. Consequently, we can bound the probability of the process being in $A$ at time $T$:
\begin{align}
    \pihat_T(A) 
    &\le \mathbb{P}(x_T \in A \mid G_0) + \mathbb{P}(G_0^c) \nonumber \\
    &\le e^{-\Omega(d)}.
\end{align}
Combining this with \eqref{eq:trap-region-target-mass}, we obtain the total variation bound for every $T \le e^{c_\star d}$:
\begin{align}
    \mathrm{TV} \big( \pihat_T, \ptar \big) 
    &\ge \ptar(A) - \pihat_T(A) \nonumber \\
    &\ge 1 - e^{-\Omega(d)},
\end{align}
which completes part (ii). 
\end{proof}

\begin{corollary}
\label{cor:monte-carlo-initialization}
For the potential of \Cref{thm:kinetic-trapping}, the conclusion of part (ii) also holds for any
deterministic initialization with $\norm{x_0},\norm{v_0}\le1.1\sqrt d$. Consequently, the mixing time of Langevin dynamics with
$\Fhat$ to any distribution within constant $\TV$ of $\ptar$ is at least $e^{cd}$.
\end{corollary}

The proof is immediate because the argument depends on the initialization only through the event $G_0$, and \Cref{lem:kinetic-exit-time} applies to any deterministic initialization $(x_0, v_0)$ satisfying $\lVert x_0 \rVert, \lVert v_0 \rVert \le 1.1\sqrt{d}$. 

\begin{remark}
\label{rem:overdamped}
The proof differs from the overdamped setting of \citet{cao2026} because inertia changes the short-time dynamics. In overdamped Langevin, the trapped process is an Ornstein-Uhlenbeck process, so the position is Gaussian and can be controlled directly. In kinetic Langevin, by contrast, the position is the time integral of an OU-type velocity process with Green's function
$$G(u) = \omega^{-1} e^{-\frac{\gamma u}{2}} \sin(\omega u), \qquad \omega = \sqrt{\alpha - \tfrac{\gamma^2}{4}},$$
and large initial velocities can produce ballistic excursions before friction dominates. The proof is therefore based on a hypoelliptic exit-time bound.
\end{remark}

\begin{remark}
\label{rem:inverse-temperature-scaling}
The construction extends verbatim to strongly log-concave targets
$\ptar\propto e^{-V}$ with $\grad V$ Lipschitz,
using concentration of Lipschitz functions of log-concave measures
\citep{bakry2014analysis}. General $\beta$ follows by rescaling $x\mapsto\sqrt\beta\,x$,
$t\mapsto t$, which maps \eqref{eq:kinetic-langevin} to the $\beta=1$ equation with a rescaled potential.
\end{remark}

\begin{remark}
Since $\Fhat$ is Lipschitz, by Girsanov's theorem any discretization $\hat x_T$ of the SDE (e.g.\ BAOAB with suitable step size; \citealp{leimkuhler2015}) satisfying
\[
\TV\bigl(\Law(\hat x_T),\Law(x_T)\bigr)\le c
\]
inherits $\TV(\Law(\hat x_T),\ptar)\ge1-c-e^{-\Omega(d)}$ by the triangle inequality.
\end{remark}

\section{Proof of the data-interpolation theorem}
\label{app:data-initialization-proof}

\begin{proof}[Proof of \Cref{thm:data-interpolation}]

Throughout, we write $z_i := x_i - \mu$ for the centered training samples and $Z := x - \mu$ for a centered draw $x \sim \ptar$, so that
\[
    z_1, \dots, z_n, Z \stackrel{\mathrm{iid}}{\sim} \mathcal{N}(0, I_d).
\]
Recall the definitions of the unperturbed energy and force:
\[
    V(x) = \frac{1}{2} \lVert x - \mu \rVert^2 \qquad \text{and} \qquad F(x) = -(x - \mu).
\]
Fix the stiffness parameter to be sufficiently large:
\begin{equation}
\label{eq:data-trap-stiffness}
    \alpha \ge \alpha_0 := \max \big\{ 40, \, \alpha_\star(\gamma, 0.10, 0.05, 1.1) \big\},
\end{equation}
with $\alpha_\star$ as in \Cref{lem:kinetic-exit-time}; note that $\alpha_0$ is a constant independent of $d$ and $n$.

Let $\psi \in C^\infty(\mathbb{R})$ satisfy $\psi \equiv 1$ on $(-\infty, 0.0225]$, $\psi \equiv 0$ on $[0.0256, \infty)$, and $0 \le \psi \le 1$. We set
\[
    \chi_i(x) := \psi \left( \frac{\lVert x - x_i \rVert^2}{d} \right),
\]
so that the cutoff function satisfies
\[
    \chi_i \equiv 1 \quad \text{on } \{\lVert x - x_i \rVert \le 0.15\sqrt{d}\} \qquad \text{and} \qquad \chi_i \equiv 0 \quad \text{on } \{\lVert x - x_i \rVert \ge 0.16\sqrt{d}\}.
\]
Define the well centers and offsets
\begin{equation}
\label{eq:data-trap-center}
    w_i := x_i - \frac{1}{\alpha} z_i, \qquad c_i := \frac{1}{2} \left( 1 - \frac{1}{\alpha} \right) \lVert z_i \rVert^2, \qquad U_i(x) := \frac{\alpha}{2} \lVert x - w_i \rVert^2 + c_i,
\end{equation}
and set the modified global energy and force as
\begin{equation}
\label{eq:data-interpolating-potential}
    \Vhat(x) := V(x) + \sum_{i=1}^n \chi_i(x) \big( U_i(x) - V(x) \big), \qquad \Fhat := -\nabla \Vhat.
\end{equation}
The offset $w_i$ displaces the well center from $x_i$ toward $\mu$ by a $1/\alpha$ fraction, and $c_i$ is the unique constant making $\Vhat$ agree with $V$ at $x_i$; these two choices are what give the exact fit asserted in part (i).

\paragraph{General-position event:}
For $i \ne j$, we have $z_i - z_j \sim \mathcal{N}(0, 2I_d)$, so $\mathbb{E} \lVert z_i - z_j \rVert = (1 + o(1))\sqrt{2d}$. By \Cref{lem:gaussian-norm-concentration}, we obtain the high-probability bounds:
\begin{align*}
    \lVert x_i - x_j \rVert &= \lVert z_i - z_j \rVert \ge 0.4\sqrt{d} \qquad \text{with probability } 1 - e^{-\Omega(d)}, \\
    \lVert z_i \rVert &\in [0.5\sqrt{d}, 2\sqrt{d}] \qquad \text{with probability } 1 - e^{-\Omega(d)}.
\end{align*}
A union bound over the $\binom{n}{2} + n = \mathrm{poly}(d)$ events gives, with probability $1 - e^{-\Omega(d)}$,
\begin{equation}
\label{eq:training-general-position-event}
    \min_{i \ne j} \lVert x_i - x_j \rVert \ge 0.4\sqrt{d} \qquad \text{and} \qquad 0.5\sqrt{d} \le \lVert z_i \rVert \le 2\sqrt{d} \quad (1 \le i \le n).
\end{equation}
Since $0.4\sqrt{d} > 2 \cdot 0.16\sqrt{d}$, the balls $B_i := \{\lVert x - x_i \rVert \le 0.16\sqrt{d}\}$ are pairwise disjoint on this event, so at most one summand in \eqref{eq:data-interpolating-potential} is nonzero at any given $x$. In particular, $\Vhat \in C^\infty(\mathbb{R}^d)$ and $\Fhat$ is conservative by construction. We work on the event \eqref{eq:training-general-position-event} for the remainder of the proof and record the consequence
\begin{equation}
\label{eq:training-offset-bound}
    \lVert x_i - w_i \rVert = \frac{1}{\alpha} \lVert z_i \rVert \le \frac{2}{\alpha}\sqrt{d} \le 0.05\sqrt{d},
\end{equation}
where the last inequality uses $\alpha \ge 40$.

\paragraph{Interpolation:}
Fix $i$. Since $\chi_i(x_i) = 1$ and $\chi_j(x_i) = 0$ for $j \ne i$, we have $\Vhat = U_i$ and $\Fhat = -\nabla U_i$ in a neighborhood of $x_i$, whence
\[
    \Fhat(x_i) = -\alpha(x_i - w_i) = -\alpha \cdot \frac{1}{\alpha} z_i = -z_i = F(x_i),
\]
and, using $\lVert x_i - w_i \rVert^2 = \lVert z_i \rVert^2 / \alpha^2$, we can compute the energy match directly:
\begin{align*}
    \Vhat(x_i) 
    &= \frac{\alpha}{2} \cdot \frac{\lVert z_i \rVert^2}{\alpha^2} + c_i \\
    &= \frac{1}{2\alpha} \lVert z_i \rVert^2 + \frac{1}{2} \left( 1 - \frac{1}{\alpha} \right) \lVert z_i \rVert^2 \\
    &= \frac{1}{2} \lVert z_i \rVert^2 = V(x_i).
\end{align*}
The empirical energy-and-force matching loss of $\Vhat$ therefore vanishes identically, which is part (i).

\paragraph{Regularity and well-posedness:}
On $B_i$, we have the distance bounds
\[
    \lVert x - \mu \rVert \le \lVert x - x_i \rVert + \lVert z_i \rVert \le 2.16\sqrt{d} \qquad \text{and, by \eqref{eq:training-offset-bound},} \qquad \lVert x - w_i \rVert \le 0.21\sqrt{d},
\]
so that the point-wise gaps are bounded by
\begin{equation}
\label{eq:data-potential-regions}
    \lvert U_i(x) - V(x) \rvert = O(\alpha d), \qquad \lVert \nabla U_i(x) - \nabla V(x) \rVert = O(\alpha\sqrt{d}) \qquad \text{on } B_i.
\end{equation}
Moreover, since $\chi_i$ is a fixed smooth function of $\lVert x - x_i \rVert^2/d$ evaluated where $\lVert x - x_i \rVert = O(\sqrt{d})$, we have
\[
    \lVert \nabla \chi_i \rVert = O(1/\sqrt{d}) \qquad \text{and} \qquad \lVert \nabla^2 \chi_i \rVert_{\mathrm{op}} = O(1/d).
\]
Differentiating \eqref{eq:data-interpolating-potential} twice and combining with \eqref{eq:data-potential-regions}, every resulting term is $O(\alpha)$. Hence,
\[
\sup_x \lVert \nabla^2 \Vhat(x) \rVert_{\mathrm{op}} \le H
\]
for a constant $H=O(\alpha)$ independent of $d$ and $n$. In particular, $\Fhat$ is globally Lipschitz and \eqref{eq:kinetic-langevin} admits a unique strong solution.

\paragraph{Target mass of the modified region:}
Let $A := \{x : \min_i \lVert x - x_i \rVert \ge 0.16\sqrt{d}\}$, so that $A^c = \bigcup_i B_i$ and $\Vhat = V$ on $A$. Fix $i$ and consider $g(Z) := \langle Z, z_i \rangle / \lVert z_i \rVert$, which is $\mathcal{N}(0, 1)$ under $\ptar$. If $x \in B_i$, then
\[
    \lvert g(Z) - \lVert z_i \rVert \rvert = \frac{\lvert \langle Z - z_i, z_i \rangle \rvert}{\lVert z_i \rVert} \le \lVert Z - z_i \rVert \le 0.16\sqrt{d}.
\]
So by \eqref{eq:training-general-position-event}, we are bounded away from the center:
\[
    g(Z) \ge \lVert z_i \rVert - 0.16\sqrt{d} \ge 0.34\sqrt{d}.
\]
The standard Gaussian tail bound therefore gives
\[
    \ptar(B_i) \le \mathbb{P} \big( \mathcal{N}(0, 1) \ge 0.34\sqrt{d} \big) \le e^{-0.05 d},
\]
and a union bound over the $n = \mathrm{poly}(d)$ balls yields
\begin{equation}
\label{eq:data-modified-region-mass}
    \ptar(A^c) \le n e^{-0.05 d} = e^{-\Omega(d)}, \qquad \text{equivalently,} \qquad \ptar(A) \ge 1 - e^{-\Omega(d)}.
\end{equation}

\paragraph{Population $L^p$ errors:}
Both $\Vhat - V$ and $\Fhat - F$ vanish on $A$. On $A^c$, they are bounded by \eqref{eq:data-potential-regions} together with the corresponding bounds on the annuli, which are of the same order. Hence, for fixed $p \ge 1$, we obtain the bounds
\begin{align*}
    \mathbb{E}_{x \sim \ptar} \big[ \lVert \Fhat(x) - F(x) \rVert^p \big]^{1/p} &\le O(\alpha\sqrt{d}) \cdot \ptar(A^c)^{1/p} \le O(\alpha\sqrt{d}) \, e^{-\Omega(d)/p} = e^{-\Omega(d)}, \\
    \mathbb{E}_{x \sim \ptar} \big[ \lvert \Vhat(x) - V(x) \rvert^p \big]^{1/p} &\le O(\alpha d) \, e^{-\Omega(d)/p} = e^{-\Omega(d)}.
\end{align*}
In both cases, these hold for all $d \ge d_0(p)$, as the polynomial prefactors are absorbed once $d$ is large relative to $p$. This is part (ii).

\paragraph{Total-variation lower bound:}
Condition on $x_0 = x_i$ and write $\Pi_i$ for the corresponding law. Since $v_0 \sim \mathcal{N}(0, I_d)$, \Cref{lem:gaussian-norm-concentration} gives
\[
    \lVert v_0 \rVert \le 1.1\sqrt{d} \qquad \text{with probability } 1 - e^{-\Omega(d)},
\]
and we work on this event. In the shifted coordinates $y_t := x_t - w_i$, we have $\Fhat(x_t) = -\alpha y_t$ for as long as $x_t \in \{\lVert x - x_i \rVert \le 0.15\sqrt{d}\}$. Thus, $(y_t, v_t)$ solves the linear SDE of \Cref{lem:kinetic-exit-time} with the confining center at the origin. By \eqref{eq:training-offset-bound},
\[
    \lVert y_0 \rVert = \lVert x_i - w_i \rVert \le 0.05\sqrt{d}, \qquad \text{and} \qquad \big\{ \lVert y \rVert \le 0.10\sqrt{d} \big\} \subseteq \big\{ \lVert x - x_i \rVert \le 0.15\sqrt{d} \big\},
\]
so \Cref{lem:kinetic-exit-time} applies with $(r, \rho_x, \rho_v) = (0.10, 0.05, 1.1)$, which is legitimate by the choice \eqref{eq:data-trap-stiffness} of $\alpha$. The exit time $\tau = \inf \{ t : \lVert y_t \rVert \ge 0.10\sqrt{d} \}$ therefore satisfies
\[
    \mathbb{P}_{\Pi_i}(\tau \le T) \le C e^{-c_\star d} \qquad \text{for every } T \le e^{c_\star d}.
\]
Since $\{\lVert x - x_i \rVert \le 0.15\sqrt{d}\} \subseteq B_i \subseteq A^c$, we obtain
\[
    \mathbb{P}_{\Pi_i}(x_T \in A) \le C e^{-c_\star d} + e^{-\Omega(d)} = e^{-\Omega(d)} \qquad \text{uniformly in } i,
\]
where the second term accounts for the velocity event. Averaging over the initialization $x_0 \sim \frac{1}{n} \sum_i \delta_{x_i}$ gives
\[
    \pihat_T(A) = \frac{1}{n} \sum_{i=1}^n \mathbb{P}_{\Pi_i}(x_T \in A) \le e^{-\Omega(d)}.
\]
Combining this with \eqref{eq:data-modified-region-mass}, we conclude
\begin{align*}
    \mathrm{TV} \big( \pihat_T, \ptar \big) 
    &\ge \ptar(A) - \pihat_T(A) \\
    &\ge 1 - e^{-\Omega(d)}
\end{align*}
for all $T \le e^{c_\star d}$ and all $d$ larger than a suitable constant, which completes part (iii).
\end{proof}

\begin{remark}
The quadratic form in \Cref{thm:data-interpolation} allows for an explicit exit-time calculation. A smooth annular-barrier construction can likewise coincide with $V$ in a neighbourhood of every training configuration preserving all local derivatives while exhibiting exponentially long metastable confinement.
\end{remark}

\begin{remark}
The trapping mechanism extends beyond the target training distribution. Suppose a training distribution $\nu$ assigns mass $e^{-\Omega(d)}$ to an admissible region $B_d$, and $\ptar(B_d)=e^{-\Omega(d)}$. A perturbation supported on $B_d$ then has exponentially small energy and force errors under $\nu$. If the learned dynamics is initialized inside a trap placed in $B_d$, the exit-time argument of \Cref{lem:kinetic-exit-time} yields
\(
\TV(\pihat_T,\ptar)\ge 1-e^{-\Omega(d)}
\)
for $T\le e^{\Omega(d)}$.
\end{remark}

\section{Proof of the equilibrium-mismatch theorem}
\label{app:equilibrium-mismatch-proof}

\begin{assumption}[Confinement]
\label{ass:confinement}
$V\in C^\infty(\R^d)$, $\sup_x\norm{\grad^2V(x)}_{\rm op}\le L$, and there exist $a\in(0,1)$ and
$b\ge0$ such that
\[
\tfrac12\inner{x}{\grad V(x)}\ \ge\ aV(x)+\tfrac{\gamma^2a(2-a)}{8(1-a)}\norm{x}^2-b
\]
for all $x\in\R^d$.
\end{assumption}

\begin{assumption}[Target concentration and extensivity]
\label{ass:target-concentration}
There exist $\bar x\in\R^d$ and constants $R\ge\theta>0$, $\kappa,\Lambda>0$, all independent of
$d$, such that
\[
\ptar\big(B(\bar x,R\sqrt d)\big)\ \ge\ 1-e^{-\kappa d},
\qquad
\sup\big\{V(x):\ \norm{x-\bar x}\le(R+3\theta)\sqrt d\big\}\ \le\ \Lambda d.
\]
\end{assumption}

Under \Cref{ass:confinement,ass:target-concentration}, we place an extensive but curvature-controlled well just outside the target's concentration region. We first establish its geometry and regularity, then show that it is invisible to target-weighted errors but captures nearly all learned equilibrium mass, and finally verify ergodicity and the arbitrary-initialization limit.

\begin{proof}[Proof of \Cref{thm:equilibrium-mismatch}]
Recall the normalization $\inf V = 0$, and write $Z := \int_{\mathbb{R}^d} e^{-\beta V}$, so that $\ptar = Z^{-1} e^{-\beta V}$. Let $\omega_d$ denote the volume of the unit ball in $\mathbb{R}^d$. We assume throughout that $d$ is large enough that $e^{-\kappa d} \le \frac{1}{2}$.

\paragraph{Construction:}
Let $\bar{x}, R, \theta, \kappa, \Lambda$ be as in \Cref{ass:target-concentration}. Choose any $x_\star \in \mathbb{R}^d$ with $\lVert x_\star - \bar{x} \rVert = (R + 2\theta)\sqrt{d}$ and set
\begin{equation}
\label{eq:off-support-pocket}
    B_{\mathrm{in}} := B\big(x_\star, \theta\sqrt{d}\big), \qquad B_{\mathrm{out}} := B\big(x_\star, 2\theta\sqrt{d}\big), \qquad \chi := \psi_{\theta\sqrt{d}, \, 2\theta\sqrt{d}}(\cdot - x_\star),
\end{equation}
with $\psi$ the radial cutoff of \Cref{lem:smooth-bump-function}, so that $\chi \equiv 1$ on $B_{\mathrm{in}}$ and $\chi \equiv 0$ off $B_{\mathrm{out}}$. Define the well depth and the learned potential by
\begin{equation}
\label{eq:equilibrium-mismatch-potential}
    \Delta := \frac{(\kappa_1 + \kappa_2)d + \log 2}{\beta}, \qquad \Vhat := V - \Delta\chi, \qquad \Fhat := -\nabla \Vhat,
\end{equation}
where
\begin{equation}
\label{eq:pocket-mass-exponent}
    \kappa_1 := \log(R/\theta) + \beta\Lambda \ge 0
\end{equation}
is a constant independent of $d$ (nonnegative because $\theta \le R$). Then $\Vhat \in C^\infty(\mathbb{R}^d)$, the field $\Fhat$ is conservative, $\Vhat = V$ off $B_{\mathrm{out}}$, and $\lVert \Vhat - V \rVert_\infty = \Delta = O(d/\beta)$.

\paragraph{Geometry of the pocket:}
If $\lVert x - x_\star \rVert \le 2\theta\sqrt{d}$, then $\lVert x - \bar{x} \rVert \ge \lVert x_\star - \bar{x} \rVert - 2\theta\sqrt{d} = R\sqrt{d}$, so
\begin{equation}
\label{eq:pocket-target-mass}
    B_{\mathrm{out}} \subseteq \big\{ x : \lVert x - \bar{x} \rVert \ge R\sqrt{d} \big\} \qquad \Longrightarrow \qquad \ptar(B_{\mathrm{out}}) \le e^{-\kappa d}
\end{equation}
by \Cref{ass:target-concentration}. Conversely, every $x \in B_{\mathrm{in}}$ satisfies $\lVert x - \bar{x} \rVert \le (R + 3\theta)\sqrt{d}$, so $V \le \Lambda d$ on $B_{\mathrm{in}}$, again by \Cref{ass:target-concentration}.

\paragraph{Regularity and well-posedness:}
By \Cref{lem:smooth-bump-function} applied with $(a,b) = (\theta\sqrt{d}, 2\theta\sqrt{d})$, we have
\[
    \lVert \nabla \chi \rVert_\infty \le \frac{C_0}{\theta\sqrt{d}} \qquad \text{and} \qquad \lVert \nabla^2 \chi \rVert_{\mathrm{op},\infty} \le \frac{C_0}{\theta^2 d} + \frac{C_0}{\theta^2 d} = \frac{2C_0}{\theta^2 d}.
\]
Hence, $\nabla^2 \Vhat = \nabla^2 V - \Delta \, \nabla^2 \chi$ obeys
\begin{align}
\label{eq:equilibrium-mismatch-hessian}
    \sup_x \lVert \nabla^2 \Vhat(x) \rVert_{\mathrm{op}} 
    &\le L + \Delta \cdot \frac{2C_0}{\theta^2 d} \nonumber \\
    &\le L + \underbrace{\frac{2C_0(\kappa_1 + \kappa_2 + 1)}{\beta\theta^2}}_{=: \, H'},
\end{align}
which is independent of $d$. In particular, $\Fhat$ is globally Lipschitz, so \eqref{eq:kinetic-langevin} driven by $\Fhat$ has a unique strong solution \citep[Thm.~5.2.1]{oksendal2013}. This is the step at which the width $\Theta(\sqrt{d})$ of the pocket is essential: it converts the extensive depth $\Delta = \Theta(d/\beta)$ into an intensive curvature.

\paragraph{Population errors (Part (i)):}
Both $\Vhat - V = -\Delta\chi$ and $\Fhat - F = \Delta \, \nabla \chi$ vanish off $B_{\mathrm{out}}$, and on $B_{\mathrm{out}}$ are bounded in magnitude by $\Delta$ and $C_0\Delta/(\theta\sqrt{d})$, respectively. Combined with \eqref{eq:pocket-target-mass}, we obtain
\begin{align*}
    \mathbb{E}_{x \sim \ptar} \big[ \lvert \Vhat(x) - V(x) \rvert^p \big]^{1/p} 
    &\le \Delta \, \ptar(B_{\mathrm{out}})^{1/p} \le O(d) \, e^{-\kappa d/p} = e^{-\Omega(d)}, \\
    \mathbb{E}_{x \sim \ptar} \big[ \lVert \Fhat(x) - F(x) \rVert^p \big]^{1/p} 
    &\le \frac{C_0\Delta}{\theta\sqrt{d}} \, \ptar(B_{\mathrm{out}})^{1/p} \le O(\sqrt{d}) \, e^{-\kappa d/p} = e^{-\Omega(d)},
\end{align*}
for all $d \ge d_0(p)$, as the polynomial prefactors are absorbed once $d$ is large relative to $p$.

Since $\ptar\big(B(\bar{x}, R\sqrt{d})\big) \ge \frac{1}{2}$ and $V \ge 0$, we can bound $Z$ from above:
\[
    Z \le 2 \int_{B(\bar{x}, R\sqrt{d})} e^{-\beta V} \le 2 \, \mathrm{vol}\big(B(\bar{x}, R\sqrt{d})\big) = 2 \, \omega_d (R\sqrt{d})^d.
\]
Meanwhile, $V \le \Lambda d$ on $B_{\mathrm{in}}$ gives $\int_{B_{\mathrm{in}}} e^{-\beta V} \ge \omega_d (\theta\sqrt{d})^d e^{-\beta\Lambda d}$. Therefore,
\begin{equation}
\label{eq:pocket-learned-mass}
    \ptar(B_{\mathrm{in}}) = \frac{1}{Z} \int_{B_{\mathrm{in}}} e^{-\beta V} \ge \frac{1}{2} \left( \frac{\theta}{R} \right)^d e^{-\beta\Lambda d} = \frac{1}{2} e^{-\kappa_1 d},
\end{equation}
with $\kappa_1$ as defined in \eqref{eq:pocket-mass-exponent}.

\paragraph{Stationary mismatch (Part (ii)):}
Write $\widehat{Z} := \int e^{-\beta\Vhat}$. Since $\Vhat = V - \Delta$ on $B_{\mathrm{in}}$,
\[
    \widehat{Z} \ge \int_{B_{\mathrm{in}}} e^{-\beta\Vhat} = e^{\beta\Delta} \int_{B_{\mathrm{in}}} e^{-\beta V} = e^{\beta\Delta} Z \, \ptar(B_{\mathrm{in}}).
\]
Since $\Vhat = V$ off $B_{\mathrm{out}}$ and $B_{\mathrm{out}}^c \subseteq B_{\mathrm{in}}^c$, we have
\begin{align*}
    \pihat(B_{\mathrm{out}}^c) 
    &= \frac{1}{\widehat{Z}} \int_{B_{\mathrm{out}}^c} e^{-\beta V} \\
    &= \frac{Z \, \ptar(B_{\mathrm{out}}^c)}{\widehat{Z}} \\
    &\le \frac{1}{e^{\beta\Delta} \, \ptar(B_{\mathrm{in}})} \\
    &\le \frac{2 e^{\kappa_1 d}}{e^{(\kappa_1 + \kappa_2)d + \log 2}}\\ 
    &= e^{-\kappa_2 d},
\end{align*}
by \eqref{eq:pocket-learned-mass} and the choice \eqref{eq:equilibrium-mismatch-potential} of $\Delta$. Combining this with $\ptar(B_{\mathrm{out}}^c) \ge 1 - e^{-\kappa d}$ from \eqref{eq:pocket-target-mass}, we obtain
\begin{equation}
\label{eq:equilibrium-total-variation}
    \mathrm{TV}(\pihat, \ptar) \ge \ptar(B_{\mathrm{out}}^c) - \pihat(B_{\mathrm{out}}^c) \ge 1 - e^{-\kappa d} - e^{-\kappa_2 d} = 1 - e^{-\Omega(d)}.
\end{equation}

It remains to identify $\pihat$ as the position marginal of the unique stationary law. The diffusion matrix of \eqref{eq:kinetic-langevin} is degenerate, acting only on velocities, so ellipticity is unavailable; the drift and diffusion fields do, however, satisfy H{\"o}rmander's bracket condition, so the generator is hypoelliptic \citep{hormander1967}. It therefore suffices to verify the Lyapunov condition of \Cref{ass:confinement} for $\Vhat$.

First, $\Vhat \ge V - \Delta \ge -\Delta$ is bounded below, and $e^{-\beta\Vhat} \le e^{\beta\Delta} e^{-\beta V}$ is integrable. Second, since $\chi \ge 0$, we have $\Vhat \le V$, hence $aV \ge a\Vhat$, and therefore
\begin{align*}
    \frac{1}{2} \langle x, \nabla \Vhat(x) \rangle 
    &= \frac{1}{2} \langle x, \nabla V(x) \rangle - \frac{\Delta}{2} \langle x, \nabla \chi(x) \rangle \\
    &\ge a \Vhat(x) + \frac{\gamma^2 a(2-a)}{8(1-a)} \lVert x \rVert^2 - b - \frac{\Delta M}{2},
\end{align*}
where $M := \sup_{x \in B_{\mathrm{out}}} \lVert x \rVert \lVert \nabla \chi \rVert_\infty < \infty$ because $\chi$ has compact support. Thus, \Cref{ass:confinement} holds for $\Vhat$ with the same parameter $a$ and with $b$ replaced by $b + \frac{1}{2}\Delta M + a\Delta$ (the last term absorbing the shift needed to renormalize $\inf \Vhat$ to zero). 

By \citet[Thm.~3.2]{mattingly2002}, the process is geometrically ergodic with unique invariant measure
\[
    \pihat(\mathrm{d}x \, \mathrm{d}v) \propto e^{-\beta(\Vhat(x) + \frac{1}{2}\lVert v \rVert^2)} \, \mathrm{d}x \, \mathrm{d}v,
\]
whose position marginal is $\pihat \propto e^{-\beta\Vhat}$. Moreover, there exist constants $C, \lambda > 0$ and a Lyapunov function $\mathcal{L}(x,v) \asymp \Vhat(x) + \lVert x \rVert^2 + \lVert v \rVert^2$ such that
\begin{equation}
\label{eq:geometric-ergodicity}
    \mathrm{TV}\big( \mathrm{Law}(x_t, v_t), \pihat \big) \le C \big( 1 + \mathbb{E} \, \mathcal{L}(x_0, v_0) \big) e^{-\lambda t}.
\end{equation}

\paragraph{Arbitrary initialization (Part (iii)):}
Fix any initial law with $\mathbb{E}[V(x_0) + \lVert x_0 \rVert^2 + \lVert v_0 \rVert^2] < \infty$. Since $\lvert \Vhat - V \rvert \le \Delta$, this is equivalent to $\mathbb{E} \, \mathcal{L}(x_0, v_0) < \infty$. Total variation between position marginals is dominated by total variation on phase space, so \eqref{eq:geometric-ergodicity} gives $\mathrm{TV}(\pihat_t, \pihat) \to 0$. By the triangle inequality,
\[
    \lvert \mathrm{TV}(\pihat_t, \ptar) - \mathrm{TV}(\pihat, \ptar) \rvert \le \mathrm{TV}(\pihat_t, \pihat),
\]
Hence by \eqref{eq:equilibrium-total-variation},
\[
    \lim_{t \to \infty} \mathrm{TV}(\pihat_t, \ptar) = \mathrm{TV}(\pihat, \ptar) \ge 1 - e^{-\Omega(d)}
\]

\end{proof}

\begin{remark}
The rate $\lambda$ in \eqref{eq:geometric-ergodicity} is typically exponentially small in $d$ for $\Vhat$, since $\pihat$ is bimodal with an exponentially deep well. Part (iii) is accordingly a purely asymptotic statement, and it is complementary to, rather than implied by, \Cref{thm:kinetic-trapping,thm:data-interpolation}: those bound finite horizons for potentials whose equilibrium is correct, whereas this one bounds the equilibrium itself.
\end{remark}

\begin{remark}
The failures exhibited by the construction are also reflected in physical observables beyond total variation which can be a stringent metric in high dimensions. \Cref{thm:kinetic-trapping,thm:data-interpolation} produce trapping and exponentially slow mixing, while \Cref{thm:equilibrium-mismatch} yields $W_2(\pihat,\ptar)=\Omega(\sqrt d)$, corresponding to an $\Omega(1)$ per-atom RMSD and hence a structurally distinct equilibrium state.
\end{remark}

\section{Proof of the sampling-guarantee theorem}
\label{app:sufficient-conditions-proof}

The first two parts of the proof compare Gibbs densities through their normalization constants. The third transfers a target log-Sobolev inequality through a bounded perturbation and then invokes hypocoercivity. We finally return to the two kinetic constructions to show that one-sided equilibrium control does not imply rapid mixing.

\begin{proof}[Proof of \Cref{thm:sampling-guarantees}]
\noindent\emph{Part (i): Uniform two-sided control.}\quad
Write the normalization constants as
\[
    Z = \int e^{-\beta V} \qquad \text{and} \qquad \widehat{Z} = \int e^{-\beta\Vhat}.
\]
From the uniform bound $\lvert \Vhat - V \rvert \le \delta$, we immediately obtain $e^{-\beta\delta} Z \le \widehat{Z} \le e^{\beta\delta} Z$. Thus, the density ratio satisfies
\[
    \frac{\mathrm{d}\pihat}{\mathrm{d}\ptar}(x) = e^{-\beta(\Vhat(x) - V(x))} \frac{Z}{\widehat{Z}} \in \big[ e^{-2\beta\delta}, \, e^{2\beta\delta} \big].
\]
Consequently, we can bound the total variation distance:
\begin{align*}
    \mathrm{TV}(\pihat, \ptar) 
    = \int (\ptar - \pihat)_+ 
    \le \big(1 - e^{-2\beta\delta}\big) \int \ptar 
    = 1 - e^{-2\beta\delta} 
    \le 2\beta\delta.
\end{align*}

\medskip\noindent\emph{Part (ii): One-sided global control.}\quad
By condition (a), we have $\Vhat - V \ge -\delta$, so $e^{-\beta(\Vhat - V)} \le e^{\beta\delta}$. By condition (b),
\begin{align*}
    \widehat{Z} 
    \ge \int_S e^{-\beta\Vhat} 
    \ge e^{-\beta\delta'} \int_S e^{-\beta V} 
    = e^{-\beta\delta'} Z \, \ptar(S) 
    \ge e^{-\beta\delta'} Z (1 - \eta)
\end{align*}
From this, we deduce
\[
    \frac{Z}{\widehat{Z}} \le \frac{e^{\beta\delta'}}{1 - \eta} \qquad \text{and} \qquad \frac{\mathrm{d}\pihat}{\mathrm{d}\ptar} \le \frac{e^{\beta(\delta + \delta')}}{1 - \eta},
\]
which proves the first claim. 

For the total variation bound, note that condition (a) also implies $\widehat{Z} \le e^{\beta\delta} Z$, hence $Z / \widehat{Z} \ge e^{-\beta\delta}$. Meanwhile, condition (b) gives $e^{-\beta(\Vhat - V)} \ge e^{-\beta\delta'}$ on the set $S$. Together, these imply $\mathrm{d}\pihat/\mathrm{d}\ptar \ge e^{-\beta(\delta + \delta')}$ on $S$. Splitting the positive part of the total variation distance over $S$ and $S^c$, we find:
\begin{align*}
    \mathrm{TV}(\pihat, \ptar) 
    &= \int_S (\ptar - \pihat)_+ + \int_{S^c} (\ptar - \pihat)_+ \\
    &\le \big( 1 - e^{-\beta(\delta + \delta')} \big) \ptar(S) + \ptar(S^c) \\
    &\le 1 - e^{-\beta(\delta + \delta')} + \eta
\end{align*}

\medskip\noindent\emph{Part (iii): Functional inequalities and relaxation.}\quad
Under condition (i), we may write the density as
\[
    \mathrm{d}\pihat = \mathcal{Z}^{-1} e^{-W} \, \mathrm{d}\ptar, \qquad \text{where} \qquad W = \beta(\Vhat - V),
\]
so that its oscillation is bounded by
\[
    \mathrm{osc}(W) := \sup W - \inf W \le 2\beta\delta.
\]
The Holley--Stroock perturbation principle \citep{holley1987} therefore implies that $\pihat$ satisfies a log-Sobolev inequality with a constant of at least
\[
    \rho e^{-\mathrm{osc}(W)} \ge \rho e^{-2\beta\delta}.
\]
Since $\nabla^2 \Vhat$ is bounded, the hypocoercivity theory of \citet{villani2009hypocoercivity} then yields exponential convergence of the kinetic Langevin dynamics driven by $\Fhat = -\nabla\Vhat$ to its invariant phase-space Gibbs measure. In particular, the convergence rate can be bounded in terms of the log-Sobolev constant of $\pihat$, the Hessian bound $H$, the friction parameter $\gamma$, and the inverse temperature $\beta$. Using the bound above, this gives an exponential relaxation rate depending only on $\rho e^{-2\beta\delta}$, $H$, $\gamma$, and $\beta$.
Thus, condition (i) guarantees exponential relaxation of the learned dynamics in addition to controlling its stationary distribution. The same conclusion does not follow from condition (ii) alone, as shown below.
\end{proof}

\begin{proposition}
\label{prop:one-sided-equilibrium-control}
The potentials of \Cref{thm:kinetic-trapping,thm:data-interpolation} both satisfy $\TV(\pihat,\ptar)\le e^{-\Omega(d)}$;
that of \Cref{thm:data-interpolation} additionally satisfies $\Vhat\ge V$ everywhere, so it meets the hypotheses of
\Cref{thm:sampling-guarantees}(ii) with $\delta=\delta'=0$, $\eta=e^{-\Omega(d)}$. Both have mixing time
$e^{\Omega(d)}$.
\end{proposition}
\begin{proof}

\noindent\emph{Data-interpolation construction:}\quad
Fix $i$ and set $q := U_i - V$. This is a quadratic function with Hessian $(\alpha - 1)\mathrm{Id} \succ 0$, and is therefore strictly convex. Its unique minimizer solves $\alpha(x - w_i) = x - \mu$, which rearranges to
\[
    (\alpha - 1)x = \alpha w_i - \mu.
\]
By \eqref{eq:data-trap-center}, we have $w_i = \mu + (1 - 1/\alpha) z_i$, so we can compute
\begin{align*}
    \alpha w_i - \mu 
    &= \alpha\mu + (\alpha - 1) z_i - \mu \\
    &= (\alpha - 1)(\mu + z_i) \\
    &= (\alpha - 1)x_i.
\end{align*}
Thus, the minimizer is exactly $x_i$, where $q(x_i) = \Vhat(x_i) - V(x_i) = 0$ by the exact-fit computation in \Cref{app:data-initialization-proof}. Hence, $U_i \ge V$ everywhere. Since $\chi_i \ge 0$ and the supports are disjoint, it follows that $\Vhat \ge V$ on all of $\mathbb{R}^d$. 

\Cref{thm:sampling-guarantees}(ii) therefore applies with $\delta = 0$, $S = A$, $\delta' = 0$, and $\eta = \ptar(A^c) \le e^{-\Omega(d)}$ from \eqref{eq:data-modified-region-mass}. This yields the bound
\[
    \mathrm{TV}(\pihat, \ptar) \le \eta = e^{-\Omega(d)}
\]

\medskip\noindent\emph{Kinetic-trapping construction:}\quad
Here, the condition $\Vhat \ge V$ fails near the origin, so we argue directly. Split $\mathbb{R}^d$ into the core $\{\lVert x \rVert \le 4\sqrt{d}\}$, the annulus $\mathcal{A} := \{4\sqrt{d} \le \lVert x \rVert \le b_\xi\sqrt{d}\}$, and the exterior $\{\lVert x \rVert \ge b_\xi\sqrt{d}\}$.

On $\mathcal{A}$, by the choice \eqref{eq:trap-stiffness} of $\alpha$ and the fact that $\lVert x - \mu \rVert \le (b_\xi + \xi)\sqrt{d}$, we have
\begin{equation}
    U(x) \ge 8\alpha d \ge \frac{1}{2} (b_\xi + \xi)^2 d \ge V(x)
\end{equation}
Since $\chi \in [0,1]$ and $\Vhat = V + \chi(U - V)$, this implies $\Vhat \ge V$ on $\mathcal{A}$. Thus, by \eqref{eq:modified-region-target-mass}, we can bound the partition function integral over the annulus:
\begin{equation}
    \int_{\mathcal{A}} e^{-\beta\Vhat} \le Z \, \ptar(\mathcal{A}) \le Z e^{-\Omega(d)}.
\end{equation}

On the core $\{\lVert x \rVert \le 4\sqrt{d}\}$, we have $\Vhat = U = \frac{\alpha}{2}\lVert x \rVert^2$, so
\begin{align*}
    \int_{\lVert x \rVert \le 4\sqrt{d}} e^{-\beta\Vhat} 
    &\le \int_{\mathbb{R}^d} e^{-\beta\alpha\lVert x \rVert^2/2} \\
    &= \left( \frac{2\pi}{\beta\alpha} \right)^{d/2} \\
    &= \alpha^{-d/2} \left( \frac{2\pi}{\beta} \right)^{d/2} \\
    &= \alpha^{-d/2} Z,
\end{align*}
which is $Z e^{-\Omega(d)}$ because $\alpha \ge 40$. 

Finally, $\Vhat = V$ on the exterior $\{\lVert x \rVert \ge b_\xi\sqrt{d}\}$, which is a set of $\ptar$-mass $1 - e^{-\Omega(d)}$ by \eqref{eq:modified-region-target-mass}. 

Therefore, $\widehat{Z} = Z(1 \pm e^{-\Omega(d)})$. The two densities are proportional on a set of $\ptar$-mass $1 - e^{-\Omega(d)}$ with a ratio of $1 \pm e^{-\Omega(d)}$, giving
\[
    \mathrm{TV}(\pihat, \ptar) \le e^{-\Omega(d)}.
\]

\medskip\noindent\emph{Mixing-time lower bounds:}
In both cases, \Cref{cor:monte-carlo-initialization} and \Cref{thm:data-interpolation}(iii) exhibit initial conditions from which the law remains at a total variation distance of $1 - e^{-\Omega(d)}$ from $\ptar$. Therefore, by the two bounds just proved and the triangle inequality, the law also remains at a distance of $1 - e^{-\Omega(d)}$ from $\pihat$ for every $T \le e^{cd}$. The mixing time is thus $e^{\Omega(d)}$.
\end{proof}

\section{Symmetry- and locality-aware hypothesis classes}
\label{app:local-invariant-models}

We consider one atomic species in a periodic cell~$\Lambda_N$ at fixed
density.  A configuration is $x=(x_1,\ldots,x_N)$.  To remove the irrelevant
non-normalizability associated with a freely translating center of mass,
probability measures may be understood either on the periodic configuration
space modulo diagonal translations or in a fixed center-of-mass gauge.  This
changes the dimension by only three, so $d=3N+O(1)$ and exponential statements
in $N$ and $d$ are equivalent.  Relative vectors use the periodic convention.
The local rules below are invariant under orthogonal transformations of all
relative vectors, equivariant under relabeling of identical atoms, and
translation invariant.  Equivalently, when the cell is transformed together
with the atoms, they have the usual $E(3)$ invariance of an interatomic
potential.

\subsection{The local invariant hypothesis class}
\label{app:local-invariant-class}

Choose a $C^\infty$ cutoff kernel that vanishes, together with all of its
derivatives, at the cutoff radius~$r_c$.  Let
$\mathcal E_i^{(L)}(x)$ denote a smooth representation of the information
reaching site~$i$ after $L$ local aggregation steps.  We require the family
$(\mathcal E_i^{(L)})_{i=1}^N$ to be permutation covariant, translation
invariant, and orthogonally covariant.  Smooth cutoff weights, rather than a
hard neighbor-list indicator, are part of the environment map.  This
convention is essential since a graph defined by the discontinuous rule
$\lVert x_i-x_j\rVert\le r_c$ does not by itself define a globally smooth
hypothesis class.

\begin{definition}[Local invariant class]
\label{def:local-invariant-class}
Let $\mathcal I$ be the set of smooth site functions that are invariant under
orthogonal transformations and permutations of the neighbors in a local
environment.  The depth-$L$ local invariant class is
\begin{equation}
\label{eq:local-invariant-class}
  \Hloc(r_c,L)
  :=
  \left\{
     \Vhat(x)=\sum_{i=1}^N
       \varepsilon\!\left(\mathcal E_i^{(L)}(x)\right)
     :\ \varepsilon\in\mathcal I
  \right\}.
\end{equation}
The same site function~$\varepsilon$ is used at every site and for every
system size.  We assume a bounded-receptive-field regime: one aggregation
step uses at most~$z$ atomic coordinate blocks and an $L$-step environment
uses at most~$m$ blocks, where $z$ and $m$ do not grow with~$N$.
\end{definition}

On a bounded-degree interaction graph,
$m\le 1+z+\cdots+z^L=O(z^L)$.  A uniform bound of this kind is essential since fixed density alone does not prevent arbitrarily many atoms from
entering a cutoff ball.  It can be justified, for example, on an admissible
configuration space with a hard-core separation condition.  Every member of
$\Hloc$ is smooth and conservative, its force is equivariant, and its energy
is size extensive and transferable across system sizes.  The class includes
the structural form used by local invariant and equivariant MLIPs.  We assume
that the target is realizable, $V\in\Hloc(r_c,L)$.

\subsection{Invariant local descriptors and concentration}
\label{app:local-descriptor-concentration}

The constructions in \Crefrange{app:kinetic-trapping-proof}{app:equilibrium-mismatch-proof}
gate a perturbation with a global distance.  Here we replace that gate by a
smooth scalar descriptor at each site.  A typical example starts from a smooth
one-hop scalar
\[
  q_j(x)=\sum_{k\ne j} f_c(x_k-x_j),
\]
where $f_c$ is a smooth, periodic, radial cutoff kernel, and averages these
quantities with smooth $L$-hop aggregation weights.

\begin{definition}[Averaged local descriptor]
\label{def:local-descriptor}
Let $w_{ij}^{(L)}(x)\ge0$ be smooth, permutation-covariant and rigid-motion
invariant weights, supported only when information can pass from $j$ to $i$
in at most $L$ aggregation steps.  Assume their normalizing sums are positive.
Set
\begin{equation}
\label{eq:local-descriptor}
  s_i(x)
  :=
  \frac{\sum_{j=1}^N w_{ij}^{(L)}(x)q_j(x)}
       {\sum_{j=1}^N w_{ij}^{(L)}(x)}.
\end{equation}
The family $(s_i)_{i=1}^N$ is permutation covariant, while each $s_i$ is a
rigid-motion-invariant scalar.  Consequently,
$\sum_i\varphi(s_i)$ is a valid invariant site-energy sum for every smooth
scalar function~$\varphi$.
\end{definition}

The averaging formula motivates the derivative scaling used below, but does
not imply it without assumptions on the aggregation weights and correlations.
We therefore state the analytic properties needed by the construction
explicitly.

\begin{lemma}[Receptive-field concentration]
\label{lem:descriptor-concentration}
Suppose that each $s_i$ depends on a coordinate-block set~$S_i$ contained in a
$2L$-hop receptive field.  Assume that there are constants
$C_z,C_z^{(2)},C_{\mathrm{ov}}>0$, independent of~$N$, such that, pointwise,
\begin{equation}
\label{eq:descriptor-gradient-bound}
  \lVert\nabla s_i\rVert
  \le C_z\sqrt{\frac{z}{m}},
  \qquad
  \lVert\nabla^2s_i\rVert_{\mathrm{op}}
  \le C_z^{(2)}\frac{z}{m},
\end{equation}
and each atomic coordinate block belongs to at most
\begin{equation}
\label{eq:descriptor-overlap-bound}
  \kappa_{\mathrm{ov}}\le C_{\mathrm{ov}}m
\end{equation}
of the sets~$S_i$.

Assume also that the target measure satisfies the following dimension-free
Lipschitz concentration inequality: for every smooth $K$-Lipschitz scalar
function~$g$ and every $t>0$,
\[
  \Prob_{x\sim\ptar}
  \bigl(\lvert g(x)-\Exp_{\ptar}g\rvert\ge t\bigr)
  \le 2\exp\!\left(-\frac{\alpha t^2}{2K^2}\right)
\]
for a constant $\alpha>0$ independent of~$N$.  Then, for every~$i$,
\begin{equation}
\label{eq:descriptor-concentration}
  \Prob_{x\sim\ptar}
  \left(\lvert s_i(x)-\Exp_{\ptar}s_i\rvert\ge t\right)
  \le
  2\exp\!\left(-\frac{\alpha m t^2}{2C_z^2z}\right).
\end{equation}
Moreover, if
\[
  n(x):=\sum_{i=1}^N\varphi(s_i(x))
\]
and $\varphi:\mathbb R\to[0,1]$ is Lipschitz, then
\begin{equation}
\label{eq:descriptor-count-concentration}
  \Prob_{x\sim\ptar}
  \left(n(x)-\Exp_{\ptar}n\ge tN\right)
  \le
  2\exp\!\left(
    -\frac{\alpha t^2N}
    {2C_{\mathrm{ov}}C_z^2z\lVert\varphi'\rVert_\infty^2}
  \right).
\end{equation}
\end{lemma}

\begin{proof}
The first bound in~\eqref{eq:descriptor-gradient-bound} makes $s_i$
$C_z\sqrt{z/m}$-Lipschitz, so~\eqref{eq:descriptor-concentration} follows directly from the
assumed concentration inequality.

For the second claim, set
$u_i:=\varphi'(s_i)\nabla s_i$, which is supported on~$S_i$.  The overlap
bound gives
\[
  \left\lVert\sum_i u_i\right\rVert^2
  \le \kappa_{\mathrm{ov}}\sum_i\lVert u_i\rVert^2.
\]
Hence
\begin{align*}
  \lVert\nabla n\rVert^2
  &\le
  \kappa_{\mathrm{ov}}\lVert\varphi'\rVert_\infty^2
  \sum_{i=1}^N\lVert\nabla s_i\rVert^2 \\
  &\le
  C_{\mathrm{ov}}C_z^2zN\lVert\varphi'\rVert_\infty^2.
\end{align*}
Thus $n$ is
$\sqrt{C_{\mathrm{ov}}C_z^2zN}\lVert\varphi'\rVert_\infty$-Lipschitz.
Applying the concentration inequality at deviation~$tN$ proves
\eqref{eq:descriptor-count-concentration}.
\end{proof}

For a descriptor separation~$\Delta s>0$, define
\begin{equation}
\label{eq:effective-receptive-dimension}
  \nu
  :=\frac{\alpha m(\Delta s)^2}{2C_z^2z}.
\end{equation}
When $m\asymp z^L$ and the remaining constants are fixed,
$\nu\asymp z^{L-1}$. 

\begin{assumption}
\label{ass:locally-recognizable-decoy}
There are constants $a_1<a_2$, $\Delta s>0$, and $\kappa>0$, independent of
$N$, and a sequence of invariant Borel sets
$\mathcal D_N$ such that
\begin{enumerate}[label=\emph{(\alph*)}]
\item
  $\operatorname{dist}([a_1,a_2],\Exp_{\ptar}s_i)\ge\Delta s$
  for every~$i$;
\item
  $s_i(x)\in[a_1,a_2]$ for every $x\in\mathcal D_N$ and every~$i$;
\item
  $\ptar(\mathcal D_N)\ge e^{-\kappa N}$.
\end{enumerate}
\end{assumption}

Condition~(c) says that the restricted free energy of the decoy set is at
most $\kappa N/\beta$ above the target free energy.  It prevents the decoy
from being superexponentially rare.  Conditions~(a)-(b), together with the
concentration lemma, make the same set exponentially atypical for the target
once a positive fraction of its sites must lie in the descriptor window.
Competing polymorphs, uniformly strained cells, and defect-rich structures are
possible physical realizations, but the theorem uses only the three stated
properties.

\subsection{Failure in the local invariant class}
\label{app:local-class-failure}

\begin{theorem}
\label{thm:local-equilibrium-mismatch}
Let $V\in\Hloc(r_c,L)$ and let
$\ptar\propto e^{-\beta V}$ be a probability measure.  Suppose that the
descriptor family satisfies \Cref{lem:descriptor-concentration},
\Cref{ass:locally-recognizable-decoy} holds, and
\begin{equation}
\label{eq:local-target-hessian}
  \sup_x\lVert\nabla^2V(x)\rVert_{\mathrm{op}}\le H_0
\end{equation}
for a constant $H_0$ independent of~$N$.

Let $\varphi\in C^\infty(\mathbb R;[0,1])$ satisfy
$\varphi\equiv1$ on $[a_1,a_2]$ and
$\varphi\equiv0$ outside $[a_1-h,a_2+h]$, where
$0<h<\Delta s/2$, with
$\lVert\varphi^{(k)}\rVert_\infty\le C_kh^{-k}$.  Define
\begin{equation}
\label{eq:local-hidden-error-exponent}
  \nu_h
  :=\frac{\alpha m(\Delta s-h)^2}{2C_z^2z}
  =\left(1-\frac{h}{\Delta s}\right)^2\nu.
\end{equation}
Choose a constant~$\Delta_1$ such that
\begin{equation}
\label{eq:local-well-depth-condition}
  \beta\Delta_1>2\kappa,
  \qquad
  \nu_h\ge\log\!\left(\frac{4\beta\Delta_1}{\kappa}\right),
\end{equation}
and set
\begin{equation}
\label{eq:local-adversarial-potential}
  \Vhat(x)
  :=V(x)-\Delta_1\sum_{i=1}^N\varphi(s_i(x)),
  \qquad
  \Fhat:=-\nabla\Vhat.
\end{equation}
Then $\Vhat\in\Hloc(r_c,2L)$, $\Vhat$ is invariant,
$\Fhat$ is equivariant, and
\begin{equation}
\label{eq:local-hessian-bound}
  \sup_x\lVert\nabla^2\Vhat(x)\rVert_{\mathrm{op}}
  \le H
\end{equation}
for a constant~$H$ independent of~$N$.  Moreover:

\begin{enumerate}[label=\emph{(\roman*)}, leftmargin=*]
\item For every fixed $p\ge1$,
\begin{equation}
\label{eq:local-energy-error}
  \frac1N\lVert\Vhat-V\rVert_{L^p(\ptar)}
  \le \Delta_1(2e^{-\nu_h})^{1/p},
\end{equation}
and
\begin{equation}
\label{eq:local-force-error}
  \left(
    \frac1{3N}\Exp_{\ptar}\lVert\Fhat-F\rVert^2
  \right)^{1/2}
  \le
  \frac{C\Delta_1C_z}{h}
  \sqrt{\frac{\kappa_{\mathrm{ov}}z}{m}}\,e^{-\nu_h/2}.
\end{equation}
In particular, since $\kappa_{\mathrm{ov}}\le C_{\mathrm{ov}}m$,
\begin{equation}
\label{eq:local-force-error-simplified}
  \left(
    \frac1{3N}\Exp_{\ptar}\lVert\Fhat-F\rVert^2
  \right)^{1/2}
  \le
  \frac{C\Delta_1C_z\sqrt{C_{\mathrm{ov}}z}}{h}
  e^{-\nu_h/2}.
\end{equation}
The displayed per-atom bounds do not grow with~$N$.

\item Write
$\widehat\pi:=\widehat Z^{-1}e^{-\beta\Vhat}$.  This measure is well
defined because $-\Delta_1N\le\Vhat-V\le0$ and the target is
normalizable.  There are constants
$c_1,c_2>0$, independent of~$N$, such that
\begin{equation}
\label{eq:local-total-variation}
  \TV(\widehat\pi,\ptar)
  \ge 1-e^{-c_1N}-2e^{-c_2N}
  =1-e^{-\Omega(N)}=1-e^{-\Omega(d)}.
\end{equation}
Normalizability is automatic on a periodic configuration space.  If the
kinetic Langevin process driven by~$\Fhat$ is ergodic, then its position
marginal converges to~$\widehat\pi$, so the same discrepancy is its
long-time sampling error.  On a connected compact position space, smoothness
is sufficient for the usual hypoelliptic ergodicity conclusion. On
$\mathbb R^d$, one may instead impose the confining hypothesis of
\Cref{ass:confinement} on~$\Vhat$.

\item $\Vhat\le V$ everywhere, and
\begin{equation}
\label{eq:local-potential-depth}
  \inf_x(\Vhat(x)-V(x))=-\Delta_1N.
\end{equation}
Consequently, the global lower bound
$\Vhat\ge V-\delta$ in \Cref{thm:sampling-guarantees} can hold only if
$\delta\ge\Delta_1N$.  Since~\eqref{eq:local-well-depth-condition} implies
$\Delta_1=\Omega(\beta^{-1})$, it fails for every
$\delta=o(N/\beta)$.
\end{enumerate}
\end{theorem}

\begin{proof}
Because $s_i$ depends on at most a $2L$-hop environment,
$\varphi(s_i)$ is an admissible depth-$2L$ site energy.  The covariance of the
family $(s_i)_i$ implies that its sum is invariant, which proves the class and
symmetry claims.

For the Hessian, write
\begin{equation}
\label{eq:local-hessian-decomposition}
  \nabla^2\Vhat
  =\nabla^2V-\Delta_1\sum_{i=1}^N B_i,
  \qquad
  B_i:=
  \varphi''(s_i)\nabla s_i\nabla s_i^\top
  +\varphi'(s_i)\nabla^2s_i.
\end{equation}
By~\eqref{eq:descriptor-gradient-bound},
\[
  \lVert B_i\rVert_{\mathrm{op}}
  \le
  C\frac{z}{m}
  \left(\frac{C_z^2}{h^2}+\frac{C_z^{(2)}}h\right).
\]
Each $B_i$ is supported on $S_i\times S_i$.  Hence, for every vector~$u$,
\[
  \left\lvert u^\top\left(\sum_iB_i\right)u\right\rvert
  \le
  \max_i\lVert B_i\rVert_{\mathrm{op}}
  \sum_i\lVert P_{S_i}u\rVert^2
  \le
  \kappa_{\mathrm{ov}}\max_i\lVert B_i\rVert_{\mathrm{op}}
  \lVert u\rVert^2.
\]
Combining this inequality with
\eqref{eq:descriptor-overlap-bound} and~\eqref{eq:local-target-hessian} gives
\[
  \sup_x\lVert\nabla^2\Vhat(x)\rVert_{\mathrm{op}}
  \le
  H_0+C\Delta_1C_{\mathrm{ov}}z
  \left(\frac{C_z^2}{h^2}+\frac{C_z^{(2)}}h\right)=:H,
\]
which is independent of~$N$.

\paragraph{Local energy and force errors:}
Let
$I_i:=\mathbf 1\{s_i\in\operatorname{supp}\varphi\}$.  
\Cref{ass:locally-recognizable-decoy}(a) and the support of~$\varphi$ imply
\[
  I_i=1
  \quad\Longrightarrow\quad
  \lvert s_i-\Exp_{\ptar}s_i\rvert\ge\Delta s-h.
\]
Therefore, by~\eqref{eq:descriptor-concentration},
\begin{equation}
\label{eq:local-gate-probability}
  \Prob_{\ptar}(I_i=1)\le2e^{-\nu_h}.
\end{equation}
Since
\begin{equation}
\label{eq:local-energy-error-indicator}
  \lvert\Vhat-V\rvert
  \le\Delta_1\sum_{i=1}^NI_i,
\end{equation}
Minkowski's inequality yields
\[
  \lVert\Vhat-V\rVert_{L^p(\ptar)}
  \le
  \Delta_1\sum_{i=1}^N\lVert I_i\rVert_{L^p(\ptar)}
  \le
  \Delta_1N(2e^{-\nu_h})^{1/p},
\]
which proves~\eqref{eq:local-energy-error}.

For the force error,
$\Fhat-F=\Delta_1\sum_i\varphi'(s_i)\nabla s_i$.  The overlap bound gives
\begin{equation}
\label{eq:local-force-overlap-bound}
  \lVert\Fhat-F\rVert^2
  \le
  \kappa_{\mathrm{ov}}\Delta_1^2
  \lVert\varphi'\rVert_\infty^2
  \sum_{i=1}^NI_i\lVert\nabla s_i\rVert^2.
\end{equation}
Taking expectations and using~\eqref{eq:local-gate-probability} and
\eqref{eq:descriptor-gradient-bound} gives
\[
  \Exp_{\ptar}\lVert\Fhat-F\rVert^2
  \le
  C\kappa_{\mathrm{ov}}\Delta_1^2h^{-2}C_z^2
  \frac{z}{m}Ne^{-\nu_h}.
\]
Dividing by~$3N$ and taking square roots proves
\eqref{eq:local-force-error}. Substituting
$\kappa_{\mathrm{ov}}\le C_{\mathrm{ov}}m$ proves
\eqref{eq:local-force-error-simplified}.

\paragraph{Global equilibrium mismatch:}
Set $n(x):=\sum_i\varphi(s_i(x))\in[0,N]$.  Exponential tilting gives, for
every measurable~$A$,
\begin{equation}
\label{eq:local-density-ratio}
  \widehat\pi(A)
  =
  \frac{\Exp_{\ptar}[\mathbf1_Ae^{\beta\Delta_1n}]}
       {\Exp_{\ptar}[e^{\beta\Delta_1n}]}.
\end{equation}
On $\mathcal D_N$, \Cref{ass:locally-recognizable-decoy}(b) gives $n=N$, so
\begin{equation}
\label{eq:local-normalizer-lower-bound}
  \Exp_{\ptar}[e^{\beta\Delta_1n}]
  \ge e^{(\beta\Delta_1-\kappa)N}.
\end{equation}

Let $\eta:=\kappa/(\beta\Delta_1)<1/2$ and
$\mathcal T_\eta:=\{n\le\eta N\}$.  From
\eqref{eq:local-density-ratio}--\eqref{eq:local-normalizer-lower-bound},
\begin{equation}
\label{eq:learned-typical-set-mass}
  \widehat\pi(\mathcal T_\eta)
  \le e^{-(\beta\Delta_1-2\kappa)N}
  =e^{-c_1N},
  \qquad c_1:=\beta\Delta_1-2\kappa>0.
\end{equation}
On the other hand,~\eqref{eq:local-gate-probability} and the second inequality in
\eqref{eq:local-well-depth-condition} imply
\begin{equation}
\label{eq:target-gate-count-mean}
  \Exp_{\ptar}n
  \le2Ne^{-\nu_h}
  \le\frac{\eta N}{2}.
\end{equation}
Applying~\eqref{eq:descriptor-count-concentration} with $t=\eta/2$ gives
\begin{equation}
\label{eq:target-typical-set-mass}
  \ptar(\mathcal T_\eta)
  \ge1-2e^{-c_2N},
  \qquad
  c_2:=
  \frac{\alpha\eta^2}
  {8C_{\mathrm{ov}}C_z^2z\lVert\varphi'\rVert_\infty^2}>0.
\end{equation}
Therefore
\[
  \TV(\widehat\pi,\ptar)
  \ge
  \ptar(\mathcal T_\eta)-\widehat\pi(\mathcal T_\eta)
  \ge1-e^{-c_1N}-2e^{-c_2N},
\]
which proves~\eqref{eq:local-total-variation}.  The statement about the long-time law is
then exactly the definition of ergodic convergence to~$\widehat\pi$.

\paragraph{Extensive depth of the perturbation:}
Because $0\le\varphi\le1$,
\[
  -\Delta_1N
  \le
  \Vhat-V=-\Delta_1\sum_i\varphi(s_i)
  \le0.
\]
\Cref{ass:locally-recognizable-decoy}(c) makes $\mathcal D_N$ nonempty, and on this set
every gate equals one.  Equality is therefore attained at the lower endpoint,
which proves~\eqref{eq:local-potential-depth} and the remaining claims.
\end{proof}

The construction is invariant under translations, identical-particle
permutations, and orthogonal transformations of the local relative vectors. Its force is equivariant under the same actions.  Higher body order or angular
resolution changes which scalar descriptor families satisfy the hypotheses,
and hence changes the constants in~\eqref{eq:effective-receptive-dimension}, but does not invalidate the
tilting argument once those hypotheses hold.

\begin{remark}
\label{rem:local-error-floor}
For fixed~$p$, the construction in the unconstrained class has errors
$e^{-\Omega(d)}$, whereas~\eqref{eq:local-energy-error} and
\eqref{eq:local-force-error-simplified} give respectively
$e^{-\Omega(\nu_h/p)}$ and $e^{-\Omega(\nu_h)}$.  Since~$\nu_h$ is independent
of~$N$ for a fixed local architecture, locality can potentially impose an $N$-independent
floor on how accurately this particular decoy can be hidden.  The global
equilibrium error nevertheless remains $1-e^{-\Omega(N)}$.
\end{remark}

\begin{remark}
In the regime $m\asymp z^L$ with all other constants fixed, achieving either
displayed per-atom error at most~$\varepsilon$ requires
$\nu_h=\Omega(\log(1/\varepsilon))$, and hence
$L=1+\Omega(\log\log(1/\varepsilon)/\log z)$. 
\end{remark}

\begin{remark}
\label{rem:system-size-transfer}
Because~$\Vhat$ is specified by one site rule, the same model can be
evaluated at different system sizes.  Its per-atom bounds in part~(i) do not
grow with~$N$, whereas the equilibrium discrepancy in part~(ii) approaches
one exponentially fast in~$N$.  Consequently, accurate validation on a small
cell does not by itself certify equilibrium sampling at deployment scale.
\end{remark}

\end{document}